%% file: main.tex
\documentclass{article}

\usepackage{microtype}
\usepackage{graphicx}
\usepackage{subfigure}
\usepackage{booktabs} 
\input{math_commands}

\usepackage{hyperref}

\usepackage{amsthm}
\usepackage{bbold}
\newtheorem{theorem}{Theorem}
\newtheorem{corollary}{Corollary}
\newtheorem{proposition}{Proposition}

\usepackage[accepted]{icml2021}

\icmltitlerunning{Improved, Deterministic Smoothing for $\ell_1$ Certified Robustness}

\begin{document}

\twocolumn[
\icmltitle{Improved, Deterministic Smoothing for $\ell_1$ Certified Robustness}

\icmlsetsymbol{equal}{*}

\begin{icmlauthorlist}
\icmlauthor{Alexander Levine}{md}
\icmlauthor{Soheil Feizi}{md}

\end{icmlauthorlist}

\icmlaffiliation{md}{Department of Computer Science, University of Maryland, College Park, Maryland, USA}

\icmlcorrespondingauthor{Alexander Levine}{alevine0@cs.umd.edu}

\icmlkeywords{Machine Learning, ICML}

\vskip 0.3in
]

\printAffiliationsAndNotice{}  

\begin{abstract}

Randomized smoothing is a general technique for computing sample-dependent robustness guarantees against adversarial attacks for deep classifiers. Prior works on randomized smoothing against $\ell_1$ adversarial attacks use additive smoothing noise and provide probabilistic robustness guarantees. In this work, we propose a non-additive and deterministic smoothing method, \textbf{D}eterministic \textbf{S}moothing with \textbf{S}plitting \textbf{N}oise (\textbf{DSSN}). To develop DSSN, we first develop SSN, a randomized method which involves generating each noisy smoothing sample by first randomly splitting the input space and then returning a representation of the center of the subdivision occupied by the input sample. In contrast to uniform additive smoothing, the SSN certification does not require the random noise components used to be independent. Thus, smoothing can be done effectively in just one dimension and can therefore be efficiently derandomized for quantized data (e.g., images). To the best of our knowledge, this is the first work to provide deterministic ``randomized smoothing" for a norm-based adversarial threat model while allowing for an arbitrary classifier (i.e., a deep model) to be used as a base classifier and without requiring an exponential number of smoothing samples. On CIFAR-10 and ImageNet datasets, we provide substantially larger $\ell_1$ robustness certificates compared to prior works, establishing a new state-of-the-art. The determinism of our method also leads to significantly faster certificate computation. Code is available at: \url{https://github.com/alevine0/smoothingSplittingNoise}.

\end{abstract}
\section{Introduction and Related Works}
Adversarial robustness in machine learning is a broad and widely-studied field which characterizes the \textit{worst-case} behavior of machine learning systems under small input perturbations \citep{szegedy2013intriguing,goodfellow2014explaining,carlini2017towards}. One area of active research is the design of {\it certifiably-robust} classifiers where, for each input $\vx$, one can compute a magnitude $\rho$, such that \textit{all} perturbed inputs $\vx'$ within a radius $\rho$ of $\vx$ are guaranteed to be classified in the same way as $\vx$. Typically, $\rho$ represents a distance in an $\ell_p$ norm: $\|\vx-\vx'\|_p \leq \rho$, for some $p$ which depends on the technique used.\footnote{Certifiably-robust classifiers for non-$\ell_p$ threat models have also been proposed including sparse ($\ell_0$) adversarial attacks \citep{levine2020robustness, lee2019tight}, Wasserstein attacks \citep{levine2020wasserstein}, geometric transformations \citep{fischer2020randomized} and patch attacks \citep{Chiang2020Certified, DBLP:conf/nips/0001F20a,Xiang2020PatchGuardPD,Zhang2020ClippedBD}. However, developing improved certified defenses under $\ell_p$ adversarial threat models remains an important problem and will be our focus in this paper.}

A variety of techniques have been proposed for certifiably $\ell_p$-robust classification \citep{wong2018provable,gowal2018effectiveness,Raghunathan2018,tjeng2018evaluating,NEURIPS2018_d04863f1,pmlr-v119-singla20a}. Among these are techniques that rely on Lipschitz analysis: if a classifier's logit functions can be shown to be Lipschitz-continuous, this immediately implies a robustness certificate \citep{Li2019PreventingGA,pmlr-v97-anil19a}. In particular, consider a classifier with logits $\{p_A, p_B, p_C, ...\}$, all of which are $c$-Lipschitz. Suppose for an input $\vx$, we have $p_A(\vx) > p_B(\vx) \geq p_C(\vx) \geq ...$. Also suppose the gap between the largest and the second largest logits is $d$ (i.e. $p_A(\vx) - p_B(\vx) = d$). The Lipschitzness implies that for all $\vx'$ such that $\|\vx-\vx'\| < d/(2c)$, $p_A(\vx')$ will still be the largest logit: in this ball, 
\begin{equation}
    p_A(\vx') > p_A(\vx) -\frac{d}{2}  \geq  p_{\text{others}}(\vx) +  \frac{d}{2}  > p_{\text{others}}(\vx'),
\end{equation}
where the first and third inequalities are due to Lipschitzness. Certification techniques based on \textit{randomized smoothing} \citep{pmlr-v97-cohen19c, salman2019provably, Zhai2020MACER, mohapatra2020higher, NEURIPS2020_77330e13, mohapatra2020higher, pmlr-v119-yang20c, lee2019tight, lecuyer2019certified, li2019certified, teng2020ell}, are, at the time of writing, the only robustness certification techniques that scale to tasks as complex as ImageNet classification. (See \citet{li2020sok} for a recent and comprehensive review and comparison of robustness certification methods.) In randomized smoothing methods, a ``base'' classifier is used to classify a large set of randomly-perturbed versions $(\vx  + \epsilon)$ of the input image $\vx$ where $\epsilon$ is drawn from a fixed distribution. The final classification is then taken as the plurality-vote of these classifications on noisy versions of the input. If samples $\vx$ and $\vx'$ are close, the distributions of $(\vx  + \epsilon)$ and $(\vx' + \epsilon)$ will substantially overlap, leading to provable robustness. \citet{salman2019provably} and \citet{levine2019certifiably} show that these certificates can in some cases be understood as certificates based on Lipschitz continuity where the \textit{expectation} of the output of the base classifier (or a function thereof) over the smoothing distribution is shown to be Lipschitz. 

 Randomized smoothing for the $\ell_1$ threat model was previously proposed by \citet{lecuyer2019certified, mohapatra2020higher, li2019certified, teng2020ell, lee2019tight} and \citet{pmlr-v119-yang20c}. \citet{pmlr-v119-yang20c} shows the best $\ell_1$ certification performance (using a certificate originally presented by \citet{lee2019tight} without experiments). These methods use {\it additive} smoothing noise and provide \textit{high-probability} certificates, with a failure rate that depends on the number of noisy samples. Outside of the setting of certified robustness, practical attacks in the $\ell_1$ threat model have additionally been studied \citep{chen2018ead}.

In this work, we propose a {\it non-additive} smoothing method for $\ell_1$-certifiable robustness on quantized data that is \textit{deterministic}. By ``quantized'' data, we mean data where each feature value occurs on a discrete level. For example, standard image files (including standard computer vision datasets, such as ImageNet and CIFAR-10) are quantized, with all pixel values belonging to the set $\{0, 1/255, 2/255, ..., 1\}$.  As \citet{carlini2017towards} notes, if a natural dataset is quantized, adversarial examples to this dataset must also be quantized (in order to be recognized/saved as valid data at all). Therefore, our assumption of quantized data is a rather loose constraint which applies to many domains considered in adversarial machine learning. We call our method {\bf D}eterministic {\bf S}moothing with {\bf S}plitting {\bf N}oise ({\bf DSSN}). DSSN produces \textit{exact} certificates, rather than high-probability ones. It also produces certificates in substantially less time than randomized smoothing because a large number of noise samples are no longer required. In addition to these benefits, the certified radii generated by DSSN are significantly larger than those of the prior state-of-the-art. 

To develop DSSN, we first propose a randomized method, Smoothing with Splitting Noise (\textbf{SSN}). Rather than simple additive noise, SSN uses ``splitting'' noise to generate a noisy input $\tilde{\vx}$: first, we generate a noise vector $\vs$ to split the input domain $[0,1]^d$ into subdivisions. Then, the noisy input $\tilde{\vx}$ is just the center of whichever sub-division $\vx$ belongs to. In contrast to prior smoothing works, this noise model is {\it non-additive}.

In contrast to additive uniform noise where the noise components ($\eps_i$'s in $\eps$) are independently distributed, in SSN, the splitting vector components ($s_i$'s in $\vs$) \textit{do not} need to be independently distributed. Thus, unlike the additive uniform smoothing where noise vectors must be drawn from a $d$-dimensional probability distribution, in SSN, the splitting vectors can be drawn from a {\it one-dimensional} distribution. In the quantized case, the splitting vector can be further reduced to a choice between a small number of elements, leading to a derandomized version of SSN (i.e. DSSN).

Below, we summarize our contributions:
\begin{itemize}
    \item We propose a novel randomized smoothing method, {\bf SSN}, for the $\ell_1$ adversarial threat model (Theorem \ref{thm:main_case_1}). 
    \item We show that \textbf{SSN} effectively requires smoothing in {\it one-dimension} (instead of $d$), thus it can be efficiently derandomized, yielding a deterministic certifiably robust classification method called \textbf{DSSN}.
    \item On ImageNet and CIFAR-10, we empirically show that DSSN significantly outperforms previous smoothing-based robustness certificates, effectively establishing a new state-of-the-art.
\end{itemize}
\subsection{Prior Works on Deterministic Smoothing}
\label{sec:prior_derandomized}
While this work is, to the best of our knowledge, the first to propose a deterministic version of a randomized smoothing algorithm to certify for a norm-based threat model without resricting the base classifier or requiring time exponential in the dimension $d$ of the input, prior deterministic ``randomized smoothing'' certificates have been proposed:
\begin{itemize}
    \item \textbf{Certificates for non-norm ($\ell_0$-like) threat models}. This includes certificates against patch adversarial attacks \citep{DBLP:conf/nips/0001F20a}; as well as poisoning attacks under a label-flipping \citep{rosenfeld2020certified} or whole-sample insertion/deletion \citep{levine2021deep} threat-model.
    These threat models are ``$\ell_0$-like'' because the attacker entirely corrupts some portion of the data, rather than just distorting it.
    \citet{DBLP:conf/nips/0001F20a} and \citet{levine2021deep} deal with this by ensuring that only a bounded fraction of base classifications can possibly be simultaneously exposed to any of the corrupted data. In the respective cases of patch adversarial attacks and poisoning attacks, it is shown that this can be done with a finite number of base classifications. \citet{rosenfeld2020certified}'s method, by contrast, is based on the randomized $\ell_0$ certificate proposed by \citet{lee2019tight}, and is discussed below.
    \item \textbf{Certificates for restricted classes of base classifiers}. This includes $k$-nearest neighbors \citep{weber2020rab} (for $\ell_2$ poisoning attacks) and linear models \citep{rosenfeld2020certified} (for label-flipping poisoning attacks). In these cases, existing randomized certificates are evaluated exactly for a restricted set of base classifier models. (\citet{pmlr-v97-cohen19c} and \citet{lee2019tight}'s methods, respectively.) It is notable that these are both poisoning certificates: in the poisoning setting, where the corrupted data is the training data, true randomized smoothing is less feasible, because it requires training very large ensemble of classifiers to achieve desired statistical properties. \citet{weber2020rab} also attempts this directly, however.
    \item \textbf{Certificates requiring time exponential in dimension $d$}. This includes, in particular, a concurrent work, \cite{kao2020deterministic}, which provides deterministic $\ell_2$ certificates. In order to be practical, this method requires that the first several layers of the network be Lipschitz-bounded by means other than smoothing. The ``smoothing'' is then applied only in a low-dimensional space. \citet{kao2020deterministic} note that this method is unlikely to scale to ImageNet.
\end{itemize}
\section{Notation}
Let $\vx$, $\vx'$ represent two points in $[0,1]^d$. We assume that our input space is bounded: this assumption holds for many applications (e.g., pixel values for image classification). If the range of values is not $[0,1]$, all dimensions can simply be scaled. A ``base'' classifier function will be denoted as $f: \R^d \rightarrow [0,1]$. In the case of a multi-class problem, this may represent a single logit.

Let $\delta := \vx' - \vx$, with components $\delta_1$, ..., $\delta_d$.  A function $p:  [0,1]^d \rightarrow [0,1]$ is said to be $c$-Lipschitz with respect to the $\ell_1$ norm iff:
\begin{equation}
   |p(\vx)- p(\vx')| \leq c \|\delta\|_1,\,\,\,\,  \forall \vx,\vx'.
\end{equation}
Let $\gU(a,b)$ represent the uniform distribution on the range $[a,b]$, and $\gU^d(a,b)$ represent a random $d$-vector, where each component is \textit{independently} uniform on  $[a,b]$.

Let $\1_{\text{(condition)}}$ represent the indicator function, and $\mathbb{1}$ be the vector $[1,1,...]^T$. In a slight abuse of notation, for $z\in \R, n \in \R^{+}$, let $z \bmod n := z - n\lfloor \frac{z}{n} \rfloor$ where $\lfloor\cdot\rfloor$ is the floor function; we will also use $\lceil\cdot\rceil$ as the ceiling function. For example, $9.5 \bmod 2 = 1.5$.

We will also discuss quantized data. We will use $q$ for the number of quantizations. Let 
\begin{equation}
    [a,b]_{(q)} := \left\{i/q \,\,\big| \,\, \lceil aq \rceil \leq i \leq  \lfloor bq \rfloor \right\}.
\end{equation}
 In particular, $[0,1]_{(q)}$ denotes the set $\{0, 1/q, 2/q, ..., (1-q)/q, 1\}$. Let $\bx,\bx'$ represent two points in $[0,1]_{(q)}^d$. A domain-quantized function $p:  [0,1]_{(q)}^d \rightarrow [0,1]$ is said to be $c$-Lipschitz with respect to the $\ell_1$ norm iff:
\begin{equation}
   |p(\bx)- p(\bx')| \leq c \|\delta\|_1,\,\,\,\,  \forall \bx,\bx' \in [0,1]_{(q)}^d,
\end{equation}
where $\delta := \bx' - \bx$. The uniform distribution on the set $[a,b]_{(q)}$ is denoted $\gU_{(q)}(a,b)$. 
\section{Prior Work on Uniform Smoothing for $\ell_1$ Robustness} \label{sec:yang}

\citet{lee2019tight} proposed an $\ell_1$ robustness certificate using uniform random noise:

\begin{theorem}[\citet{lee2019tight}] \label{thm:uniform}For any $f: \R^d   \rightarrow [0,1]$ and parameter $\lambda \in \R^{+}$, define:
\begin{equation}
    p(\vx) := \mathop{\E}_{\epsilon \sim \gU^d(-\lambda,\lambda)} \left[f(\vx + \epsilon)\right].
\end{equation}
Then, $p(.)$ is $1/(2\lambda)$-Lipschitz with respect to the $\ell_1$ norm.
\end{theorem}

\citet{pmlr-v119-yang20c} later provided a theoretical justification for the uniform distribution being optimal among {\it additive} noise distributions for certifying $\ell_1$ robustness\footnote{More precisely, \citet{pmlr-v119-yang20c} suggested that distributions with $d$-cubic level sets are optimal for $\ell_1$ robustness.}. \citet{pmlr-v119-yang20c} also provided experimental results on CIFAR-10 and ImageNet which before our work were the state-of-the-art $\ell_1$ robustness certificates. 

Following \citet{pmlr-v97-cohen19c}, \citet{pmlr-v119-yang20c} applied the smoothing method to a ``hard'' (in \citet{salman2019provably}'s terminology)  base classifier. That is, if the base classifier returns the class $c$ on input  $\vx+\epsilon$, then  $f_c(\vx+\epsilon) = 1$, otherwise $f_c(\vx+\epsilon) = 0$. 
Also following \citet{pmlr-v97-cohen19c}, in order to apply the certificate in practice, \citet{pmlr-v119-yang20c} first takes $N_0 =64$ samples to estimate the plurality class $A$, and then uses  $N = 100,000$ samples to lower-bound $p_A(\vx)$ (the fraction of noisy samples  classified as $A$) with high probability. The other smoothed logit values ($p_B(\vx),$ etc.) can then all be assumed to be $\leq 1-p_A(\vx)$. This  approach has the benefit that each logit does not require an independent statistical bound, and thus reduces the estimation error, but has the drawback that certificates are impossible if $p_A(\vx) \leq 0.5$, creating a gap between the clean accuracy of the smoothed classifier and the certified accuracy near $\rho = 0$.

We note that the stated Theorem \ref{thm:uniform} is slightly more general than the originally stated version by \citet{lee2019tight}: the original version assumed that only $p_A(\vx)$ is available, as in the above estimation scheme, and therefore just gave the $\ell_1$ radius in which $p_A(\vx')$ is guaranteed to remain $\geq 0.5$. For completeness, we provide a proof of the more general form (Theorem \ref{thm:uniform}) in the appendix. 

In this work, we show that by using \textit{deterministic smoothing} with \textit{non-additive noise}, improved certificates can be achieved, because we (i) avoid the statistical issues presented above (by estimating all smoothed logits \textit{exactly}), and (ii) improve the performance of the base classifier itself.
\section{Our Proposed Method}
In this paper, we describe a new method, Smoothing with Splitting Noise (\textbf{SSN}), for certifiable robustness against $\ell_1$ adversarial attacks. In this method, for each component $x_i$ of $\vx$, we randomly split the interval $[0,1]$ into sub-intervals. The noised value $\tilde{x}_i$ is the middle of the sub-interval that contains $x_i$. We will show that this method corresponds closely to the uniform noise method, and so we continue to use the parameter $\lambda$. The precise correspondence will become clear in Section \ref{sec:marg_distrib}: however, for now, $\lambda$ can be interpreted as controlling (the inverse of) the frequency with which the interval $[0,1]$ is split into sub-intervals. We will show that this method, unlike the additive uniform noise method, can be efficiently derandomized. For simplicity, we will first consider the case corresponding to $\lambda \geq 0.5$, in which at most two sub-intervals are created, and present the general case later.
\begin{figure}[t]
    \centering
    \includegraphics[width=0.40\textwidth]{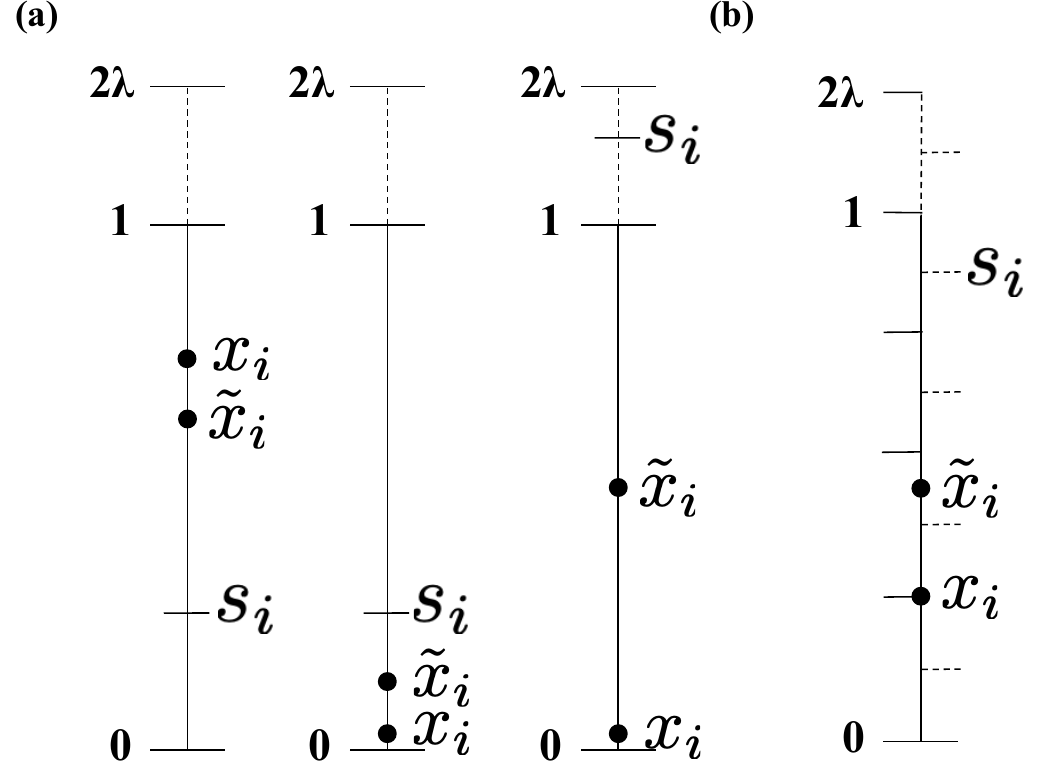}
    \caption{(a) Definition of $\tilde{\vx}$ in the $\lambda \geq 0.5$ case. If $s_i \in [0,1)$, then it ``splits'' the interval $[0,1]$: $\tilde{x}_i$ is the center of whichever sub-interval $x_i$ occurs in. If $s_i > 1$, $\tilde{x}_i = 0.5$, and no information about the original pixel is kept. (b) An example of $\tilde{\bx}$ in the \textit{quantized} $\lambda \geq 0.5$ case. Here, $q=4$ and $2\lambda = 5/4 $. We see that $\bx_i = 1/4$ lies directly on a quantization level, while $s_i = 7/8$ lies on a half-step between quantization levels. We choose $s_i$ to lie on ``half-steps'' for the sake of symmetry: the range of $\tilde{\bx}_i$ is symmetrical around $1/2$.}
    \label{fig:split_randomization_1}
\end{figure}
\begin{theorem}[$\lambda \geq 0.5$ Case]  \label{thm:main_case_1}
For any $f: \R^d   \rightarrow [0,1]$, and $\lambda \geq 0.5$ let $\vs \in [0,2\lambda]^d$ be a random variable with a fixed distribution such that:
\begin{equation}
      s_i \sim \gU(0,2\lambda), \,\,\,\, \forall i.
\end{equation}
Note that the components $s_1, ..., s_d$ are \textbf{not} required to be  distributed independently from each other. Then, define:
\begin{align}
\tilde{x}_i &:=
\frac{\min{(s_i,1)} + \1_{x_i > s_i}}{2}\
,\,\,\,\, \forall i \label{eq:x_tilde_def}\\
p(\vx) &:=\mathop{\E}_{\vs}\left[ f(\tilde{\vx})\right].
\end{align}
Then $p(\cdot)$ is $1/(2\lambda)$-Lipschitz with respect to the $\ell_1$ norm.
\end{theorem}
To understand the distribution of $\tilde{x}_i$, we can view $s_i$ as ``splitting'' the interval $[0,1]$ into two sub-intervals, $[0,s_i]$ and $(s_i, 1]$. $\tilde{x}_i$ is then the middle of whichever sub-interval contains $x_i$. If $s_i \geq 1$, then the interval $[0,1]$ is not split, and $\tilde{x}_i$ assumes the value of the middle of the entire interval ( $= 1/2$): see Figure \ref{fig:split_randomization_1}-a.

\begin{proof}
Consider two arbitrary points $\vx, \vx'$ where $\delta:=\vx'-\vx$.  Note that $ \max(x_i, x'_i) - \min(x_i, x'_i) = |x_i'-x_i|=  |\delta_i|$. For a fixed vector $\vs$, additionally note that
$\tilde{x}_i = \tilde{x}'_i$ unless $s_i$ falls between $x_i$ and $x'_i$ (i.e., unless  $\min(x_i, x'_i) \leq s_i  <  \max(x_i, x'_i)$).  Therefore:
\begin{equation}
    \Pr_\vs [\tilde{x}_i \neq \tilde{x}'_i] = \frac{|\delta_i|}{2\lambda}.
\end{equation}
By union bound:
\begin{equation} \label{eq:union_bound}
\begin{split}
      \Pr_\vs [\tilde{\vx} \neq \tilde{\vx'}]& =   \Pr_\vs \left[\bigcup_{i=1}^d \tilde{x}_i \neq \tilde{x}'_i\right] \\
      &\leq \sum_{i=1}^d  \frac{|\delta_i|}{2\lambda} = \frac{\|\delta\|_1}{2\lambda}  
\end{split}
\end{equation}
Then:
\begin{equation}
    \begin{split}
        &|p(\vx)- p(\vx')|\\ &= \left|\mathop{\E}_{\vs}\left[ f(\tilde{\vx})\right] - \mathop{\E}_{\vs}\left[ f(\tilde{\vx}')\right]\right| \\&=
         \left|\mathop{\E}_{\vs}\left[ f(\tilde{\vx}) - f(\tilde{\vx}')\right]\right| \\&=
          \Bigg|\Pr_\vs [\tilde{\vx} \neq \tilde{\vx}']\mathop{\E}_{\vs}
          \left[ f(\tilde{\vx}) - f(\tilde{\vx}') | \tilde{\vx} \neq \tilde{\vx}'\right]
          \\&+ 
          \Pr_\vs [\tilde{\vx} = \tilde{\vx}']
          \mathop{\E}_{\vs}\left[ f(\tilde{\vx}) - f(\tilde{\vx}') | \tilde{\vx} = \tilde{\vx}'\right]
          \Bigg| \\
              \end{split}
\end{equation}
Because $\mathop{\E}_{\vs}\left[ f(\tilde{\vx}) - f(\tilde{\vx}') | \tilde{\vx} = \tilde{\vx}'\right]$ is zero, we have:
\begin{equation}
    \begin{split}
          & |p(\vx)- p(\vx')|\\
         & =\Pr_\vs [\tilde{\vx} \neq \tilde{\vx}'] \left|\mathop{\E}_{\vs}
          \left[ f(\tilde{\vx}) - f(\tilde{\vx}') | \tilde{\vx} \neq \tilde{\vx}'\right]
         \right| \\
         &\leq \frac{\|\delta\|_1}{2\lambda}\cdot 1
    \end{split}
\end{equation}
where in the final step, we used Equation \ref{eq:union_bound}, as well as the assumption that $f(\cdot) \in [0,1]$. Thus, by the definition of Lipschitz-continuity,  $p$ is  $1/(2\lambda)$-Lipschitz with respect to the $\ell_1$ norm.
\end{proof}
It is important that we do \textbf{not} require that $s_i$'s be independent. (Note the union bound in Equation \ref{eq:union_bound}: the inequality holds regardless of the joint distribution of the components of $\vs$, as long as each $s_i$ is uniform.) This allows us to develop a deterministic smoothing method below.
\subsection{Deterministic SSN (DSSN)} \label{sec:DSSN}
If SSN is applied to quantized data (e.g. images), we can use the fact that the noise vector $\vs$ in Theorem \ref{thm:main_case_1} is \textit{not} required to have independently-distributed components to derive an efficient derandomization of the algorithm. In order to accomplish this, we first develop a quantized version of the SSN method, using input $\bx \in [0,1]^d_q$ (i.e. $\bx$ is a vector whose components belong to $\{0,1/q,...,1\}$). To do this, we simply choose each of our splitting values $s_i$ to be on one of the half-steps between possible quantized input values:  $\vs \in \left[0,2\lambda-1/q\right]_{(q)}^d + \mathbb{1}/(2q)$. We also require that $2\lambda$ is a multiple of $1/q$ (in experiments, when comparing to randomized methods with continuous $\lambda$, we use $\lambda' = \floor{2\lambda q}/{2q}$.) See Figure \ref{fig:split_randomization_1}-b.
\begin{corollary}[$\lambda \geq 0.5$ Case]  \label{thm:main_case_1_quantized}
For any $f: \R^d   \rightarrow [0,1]$, and $\lambda \geq 0.5$ (with $2\lambda$ a multiple of $1/q$), let $\vs \in \left[0,2\lambda-1/q\right]_{(q)}^d + \mathbb{1}/(2q)$ be a random variable with a fixed distribution such that:
\begin{equation}
      s_i \sim \gU_{(q)}\left(0,2\lambda-1/q\right) + 1/(2q), \,\,\,\, \forall i.
\end{equation}
Note that the components $s_1, ..., s_d$ are \textbf{not} required to be  distributed independently from each other. Then, define:
\begin{align}
\tilde{\bx}_i &:= \frac{\min(s_i,1) + \1_{\bx_i > s_i}}{2},\,\,\,\, \forall i \label{eq:x_tilde_def_quantized}\\
p(\bx) &:=\mathop{\E}_{\vs}\left[ f(\tilde{\bx})\right].
\end{align}
Then, $p(.)$ is $1/(2\lambda)$-Lipschitz with respect to the $\ell_1$ norm on the quantized domain $\bx \in [0,1]^d_{(q)}$.
\end{corollary}
\begin{proof}
Consider two arbitrary quantized points $\bx, \bx'$ where $\delta=\bx'-\bx$.  Again note that $ \max(\bx_i, \bx'_i) - \min(\bx_i, \bx'_i) = |\bx_i'-\bx_i|=  |\delta_i|$. For a fixed vector $\vs$, additionally note that
$\tilde{\bx}_i = \tilde{\bx}'_i$ unless $s_i$ falls between $\bx_i$ and $\bx'_i$ (i.e., unless  $\min(\bx_i, \bx'_i) \leq s_i  <  \max(\bx_i, \bx'_i)$). 
Note that $\delta_i$ must be a multiple of $1/q$, and that there are exactly $q\cdot |\delta_i|$ discrete values that $s_i$ can take such that the condition $\min(\bx_i, \bx'_i) \leq s_i  <  \max(\bx_i, \bx'_i)$ holds. This is out of $2\lambda q$ possible values over which $s_i$ is uniformly distributed. Thus, we have:
\begin{equation}
    \Pr_\vs [\tilde{\bx}_i \neq \tilde{\bx'}_i] = \frac{|\delta_i|}{2\lambda}
\end{equation}
The rest of the proof proceeds as in the continuous case (Theorem \ref{thm:main_case_1}).
\end{proof}
If we required that $s_i$'s be independent, an exact computation of $p(\bx)$ would have required evaluating $(2\lambda q)^d$ possible values of $\vs$. This is not practical for large $d$. However, because we do not have this independence requirement, we can avoid this exponential factor. To do this, we first choose a single scalar splitting value $s_{\text{base}}$: each $s_i$ is then simply a constant offset of  $s_{\text{base}}$. We proceed as follows:

First, before the classifier is ever used, we choose a single, fixed, arbitrary vector $\bv \in [0,2\lambda-1/q]_{(q)}^d$. In practice, $\bv$ is generated pseudorandomly when the classifier is trained, and the seed is stored with the classifier so that the same $\bv$ is used whenever the classifier is used. Then, at test time, we sample a scalar variable as:
\begin{equation}
        s_{\text{base}} \sim \gU_{(q)}(0,2\lambda-1/q) + 1/(2q).
\end{equation}
Then, we generate each $s_i$ by simply adding the base variable $s_{base}$ to $v_i$: 
\begin{equation}
       \forall i,   \,\,\,\,\,s_i :=  (s_{\text{base}} + v_i) \mod 2\lambda
\end{equation}
Note that the marginal distribution of each $s_i$ is $s_i \sim \gU_{(q)}\left(0,2\lambda-1/q\right) + 1/(2q)$, which is sufficient for our provable robustness guarantee. In this scheme, the only source of randomness at test time is the single random scalar $s_{\text{base}}$, which takes on one of $2\lambda q$ values. We can therefore evaluate the exact value of $p(\bx)$ by simply evaluating $f(\tilde{\bx})$ a total of $2\lambda q$ times, for each possible value of $s_{\text{base}}$. Essentially, by removing the independence requirement, the splitting method allows us to replace a $d$-dimensional noise distribution with a {\it one}-dimensional noise distribution. In quantized domains, this allows us to efficiently derandomize the SSN method without requiring exponential time. We call this resulting deterministic method \textbf{DSSN}.

One may wonder why we do not simply use $s_1 = s_2 = s_3...= s_d$. While this can work, it leads to some undesirable properties when $\lambda > 0.5$. In particular, note that with probability $(2\lambda -1)$, we would have all splitting values $s_i > 1$. This means that every element $\tilde{x}_i$ would be 0.5. In other words, with probability $(2\lambda -1)/ (2\lambda)$, $\tilde{\bx} = 0.5 \cdot \mathbb{1}$. This restricts the expressivity of the smoothed classifier:

\begin{equation}
    p(\bx) = \frac{2\lambda -1}{2\lambda} f(0.5 \cdot \mathbb{1}) +\frac{1}{2\lambda}\mathop{\E}_{\vs < 1}\left[ f(\tilde{\bx})\right]. \label{eq:all_s_equal}
\end{equation}
This is the sum of a constant, and a function bounded in $[0,1/(2\lambda)]$. Clearly, this is undesirable. By contrast, if we use an offset vector $\bv$ as described above, not every component will have $s_i > 1$ simultaneously. This means that $\tilde{\vx}$ will continue to be sufficiently expressive over the entire distribution of $s_{\text{base}}$.

\subsection{Relationship to Uniform Additive Smoothing}  \label{sec:uas_compare}
In this section, we explain the relationship between SSN and uniform additive smoothing \citep{pmlr-v119-yang20c} with two main objectives:
\begin{enumerate}
    \item We show that, for each element $x_i$, the \textit{marginal} distributions of the noisy element $\tilde{x}_i$ of SSN and the noisy element $(x_i + \epsilon_i)$ of uniform additive smoothing are directly related to one another. However we show that, for large $\lambda$, the distribution of uniform additive smoothing $(x_i + \epsilon_i)$ has an undesirable property which SSN avoids. This creates large empirical improvements in certified robustness using SSN, demonstrating an additional advantage to our method separate from derandomization.
    \item We show that additive uniform noise does \textit{not} produce correct certificates when using arbitrary joint distributions of $\epsilon$. This means that it cannot be easily derandomized in the way that SSN can.
\end{enumerate}
\subsubsection{Relationship between Marginal Distributions of  $\tilde{x}_i$ and $(x_i + \epsilon_i)$ } \label{sec:marg_distrib}
To see the relationship between uniform additive smoothing and SSN, we break the marginal distributions of each component of noised samples into cases (assuming $\lambda \geq 0.5$):
\begin{equation} \label{eq:uniform_additive_cases}
      x_i + \epsilon_i \sim  \begin{cases}
                 \gU(x_i-\lambda,1-\lambda) &\text{ w. prob. } \frac{1-x_i}{2\lambda} \\    \gU(1-\lambda,\lambda) &\text{ w. prob. } \frac{2\lambda - 1}{2\lambda} \\
                 \gU(\lambda,x_i+\lambda) &\text{ w. prob. } \frac{x_i}{2\lambda}
                        \\
                    \end{cases}
\end{equation}
\begin{equation}
       \tilde{x}_i \sim \begin{cases}
                     \,\,\,\frac{\gU(x_i,1)}{2} &\text{ w. prob. } \frac{1-x_i}{2\lambda}\\
                     \,\,\,\,\,\,\,\,\,\, \frac{1}{2} &\text{ w. prob. } \frac{2\lambda - 1}{2\lambda} \\ 
                     \frac{\gU(1,x_i+1) }{2} &\text{ w. prob. } \frac{x_i}{2\lambda}
                    \end{cases}
\end{equation}

We can see that there is a clear correspondence (which also justifies our re-use of the parameter $\lambda$.) In particular, we can convert the marginal distribution of uniform additive noise to the  marginal distribution of SSN by applying a simple mapping: $\tilde{x}_i \sim g(x_i + \epsilon_i )$ where:
\begin{equation} \label{eq:noise_mapping}
    g(z) := \begin{cases}
    \frac{z+\lambda}{2} \quad &\text{     if }  z < 1-\lambda\\
    \frac{1}{2} \quad &\text{     if }   1-\lambda<z< \lambda\\
    \frac{z-\lambda + 1}{2} \quad &\text{     if }  z>\lambda \\
    \end{cases}
\end{equation}
For $\lambda = 0.5$, this is a simple affine transformation: 
\begin{equation}
    \tilde{x}_i \sim 1/2 (x_i + \epsilon_i) + 1/4 \label{eq:equivalence_lambda_half}
\end{equation} 
In other words, in the case of $\lambda = 0.5$, $\tilde{x}_i$ is also uniformly distributed. However, for $\lambda > 0.5$, Equation \ref{eq:uniform_additive_cases} reveals an unusual and undesirable property of using uniform additive noise: \textit{regardless of the value of $x_i$}, there is always a fixed probability $\frac{2\lambda-1}{2\lambda}$ that the smoothed value $x_i + \epsilon_i$ is uniform on the interval $[1-\lambda, \lambda]$.  Furthermore, this constant probability represents the only case in which $(x_i + \epsilon_i)$ can assume values in this interval. These values therefore carry no information about $x_i$ and are all equivalent to each other. However, if $\lambda$ is large, this range dominates the total range of values of  $x_i + \epsilon_i$ which are observed (See Figure \ref{fig:compare_representations}-b.)
\begin{figure}[t]
    \centering
    \includegraphics[width=0.30\textwidth]{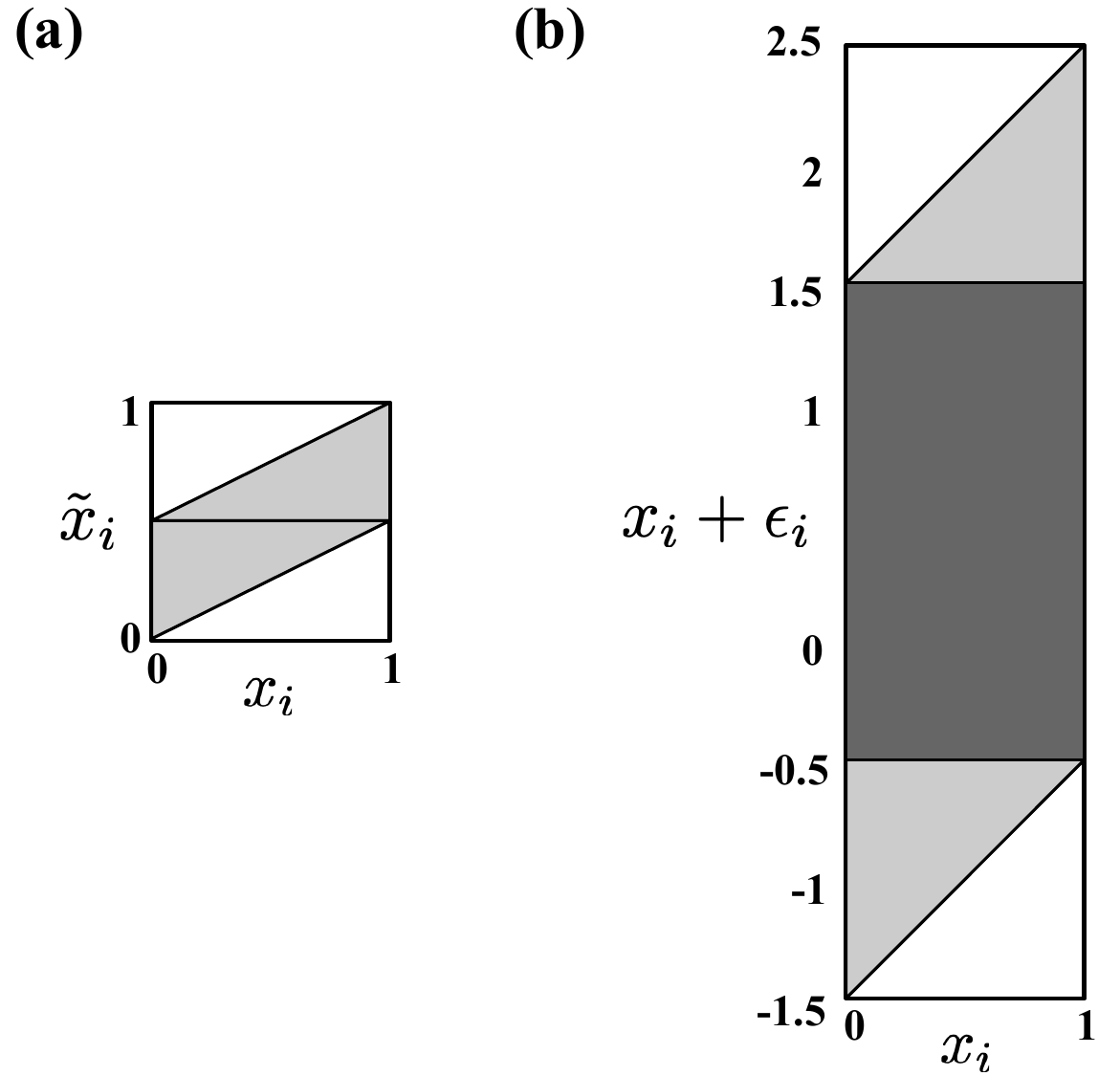}
    \caption{Range of noise values possible for each sample feature $x_i$, under (a) SSN, for any $\lambda \geq 0.5$ and (b) uniform additive smoothing, $\lambda = 1.5$. Possible pairs of clean and noise values are shown in grey (both light and dark). In uniform additive smoothing, note that all values of $x_i+\epsilon_i$ in the range [-0.5,1.5], shown in dark grey, can correspond to \textit{any} value of $x_i$. This means that these values of $x_i+\epsilon_i$ carry no information about $x_i$ whatsoever. By contrast, using SSN, only the value $\tilde{x}_i = 1/2$ has this property.}
    \label{fig:compare_representations}
\end{figure}
\begin{figure}[h]
    \centering
    \includegraphics[width=0.48\textwidth]{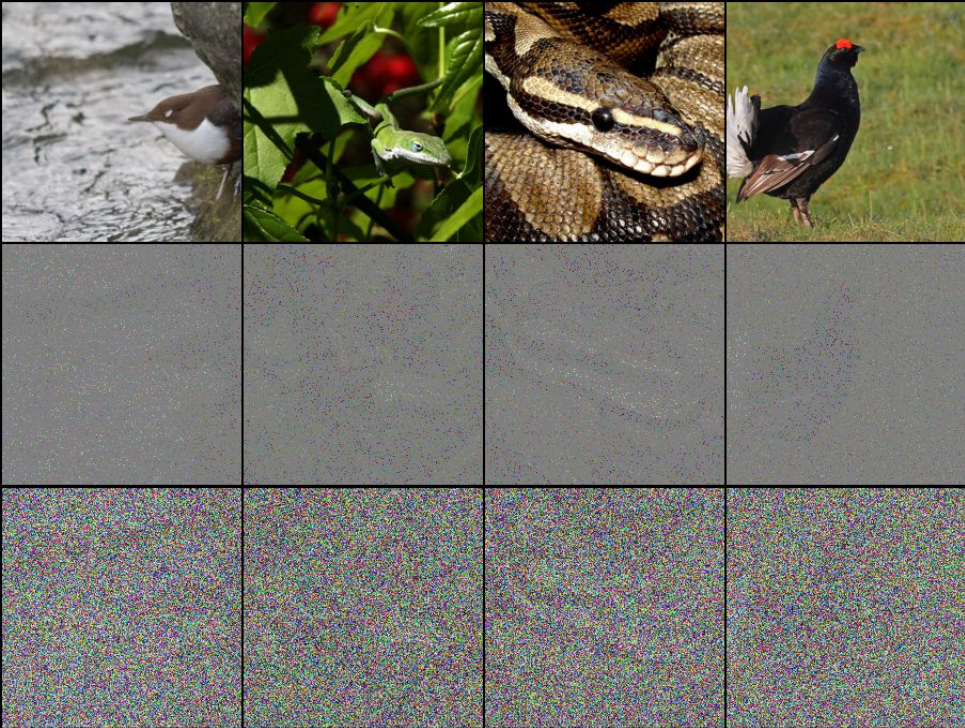}
    \caption{ImageNet images (top), under DSSN noise (middle) and uniform additive
noise (bottom). In all cases, $\sigma (=\lambda/\sqrt{3}) = 3.5$. The images with additive noise are scaled to fit in the $[0, 1]$ range in order to be displayed. The
images appear to be somewhat more visually discernible under DSSN noise, compared to additive noise.}
    \label{fig:noise_images}
\end{figure}
By contrast, in SSN, while there is still a fixed  $\frac{2\lambda-1}{2\lambda}$ probability that the smoothed component $\tilde{x}_i$ assumes a ``no information'' value, this value is always \textit{fixed} ($\tilde{x}_i=1/2$). Empirically, this dramatically improves performance when $\lambda$ is large. Intuitively, this is because when using uniform additive smoothing, the base classifier must \textit{learn to ignore} a very wide range of values (all values in the interval $[1-\lambda, \lambda]$) while in SSN, the base classifier only needs to learn to ignore a specific constant ``no information'' value $1/2$.\footnote{Note that this use of a ``no information'' value bears some similarity to the ``ablation'' value in \citet{levine2020robustness}, a randomized smoothing defense for $\ell_0$ adversarial attacks} Figure \ref{fig:compare_representations} compares the two representations schematically, and Figure \ref{fig:noise_images} compares the two noise representations visually.
\begin{figure}
    \centering
    \includegraphics[width=0.48\textwidth]{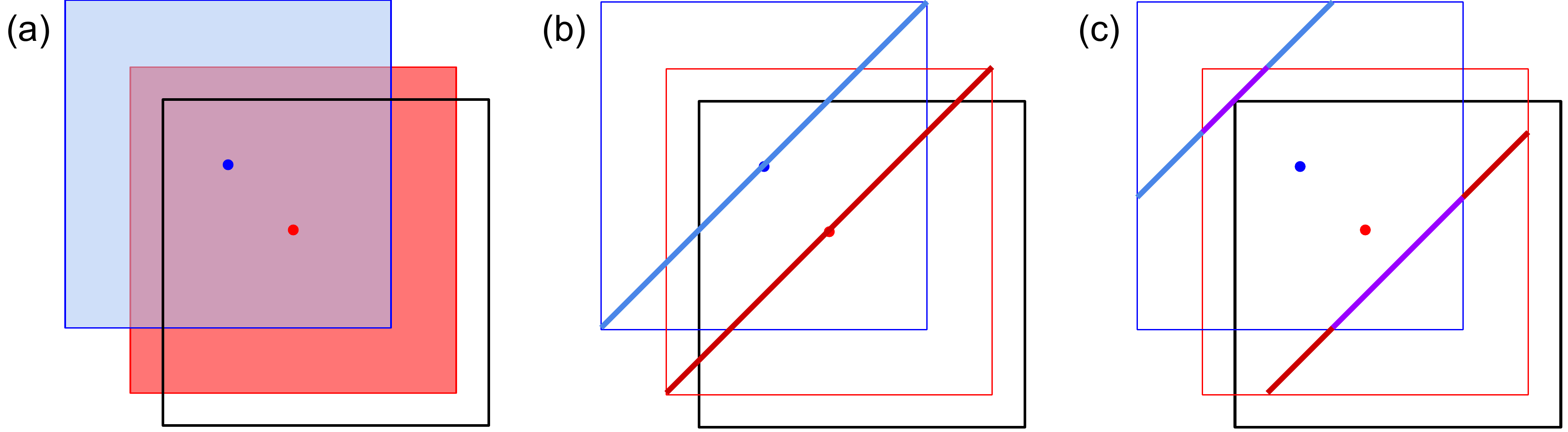}
    \caption{Comparison of independent uniform additive noise, correlated uniform additive noise, and correlated SSN, in $\R^2$ for $\lambda = 0.5$. In all figures, the blue and red points represent points $\vx$ and $\vx'$ and the black border represents the range $[0,1]^2$. (a) Distributions of $\vx+\epsilon$ and $\vx'+\epsilon$ for independent uniform additive noise. The robustness guarantee relies on the significant overlap of the shaded regions, representing the sampled distributions. Note that by Equation \ref{eq:equivalence_lambda_half}, these are also the distributions of for $2\tilde{\vx} - \mathbb{1}/2$ and $2\tilde{\vx}' - \mathbb{1}/2$ using SSN with $s_1$ and $s_2$ distributed independently. (b) Using correlated additive noise  ($\epsilon_1 = \epsilon_2$) does \textit{not} produce an effective robustness certificate: the sampled distributions $\vx+\epsilon$ and $\vx'+\epsilon$ (blue and red lines) do not overlap. (c) Using correlated splitting noise ($s_1 = s_2$) produces an effective robustness certificate, because distributions of $\tilde{\vx}$ and $\tilde{\vx}'$ overlap significantly. Here, for consistency in scaling, we show the distributions of  $2\tilde{\vx} - \mathbb{1}/2$ and $2\tilde{\vx}' - \mathbb{1}/2$ (blue line and red line), with the overlap shown as purple. Note that this is a {\it one-dimensional} smoothing distribution, and therefore can be efficiently derandomized. }
    \label{fig:l1_indep}
\end{figure}
\subsubsection{Can Additive Uniform Noise Be Derandomized?}
As shown above, in the $\lambda=0.5$ case, SSN leads to marginal distributions which are simple affine transformations of the marginal distributions of the uniform additive smoothing. One might then wonder whether we can derandomize additive uniform noise in a way similar to DSSN. In particular, one might wonder whether arbitrary joint distributions of $\epsilon$ can be used to generate valid robustness certificates with uniform additive smoothing, in the same way that arbitrary joint distributions of $\vs$ can be used with SSN. It turns out that this is not the case. We provide a counterexample (for $\lambda = 0.5$) below:
\begin{proposition} \label{prop:uniform_broken}
There exists a base classifier  $f: \R^2   \rightarrow [0,1]$ and a joint probability distribution $\gD$, such that $\epsilon_1,\epsilon_2 \sim \gD$ has marginals  $\epsilon_1 \sim \gU(-0.5,0.5)$ and $\epsilon_2 \sim \gU(-0.5,0.5)$ where for
\begin{equation}
 p(\vx) := \mathop{\E}_{\epsilon \sim\gD } \left[f(\vx + \epsilon)\right],
\end{equation}
$p(.)$ is \textbf{not} $1$-Lipschitz with respect to the $\ell_1$ norm.
\end{proposition}
\begin{proof}
Consider the base classifier $f(\vz) := \1_{z_1 > 0.4 + z_2 }$, and let $\epsilon$ be distributed as $\epsilon_1 \sim \gU(-0.5,0.5)$ and $\epsilon_2 = \epsilon_1$. Consider the points $\vx = [0.8,0.2]^T$ and $\vx' = [0.6,0,4]^T$.  Note that $\|\delta\|_1 = 0.4$. However, 
\begin{equation}
    \begin{split}
        p(\vx) = \mathop{\E}_{\epsilon}[f(\vx+\epsilon)] = \mathop{\E}_{\epsilon_1}[f(.8+\epsilon_1, .2+\epsilon_1) ] = 1 \\
         p(\vx') = \mathop{\E}_{\epsilon}[f(\vx'+\epsilon)] = \mathop{\E}_{\epsilon_1}[f(.6+\epsilon_1, .4+\epsilon_1) ] = 0 \\      
    \end{split}
\end{equation}
Thus, $|p(\vx)-p(\vx')| > \|\delta\|_1$.
\end{proof}

\begin{figure}[h]
    \centering
    \includegraphics[width=0.45\textwidth]{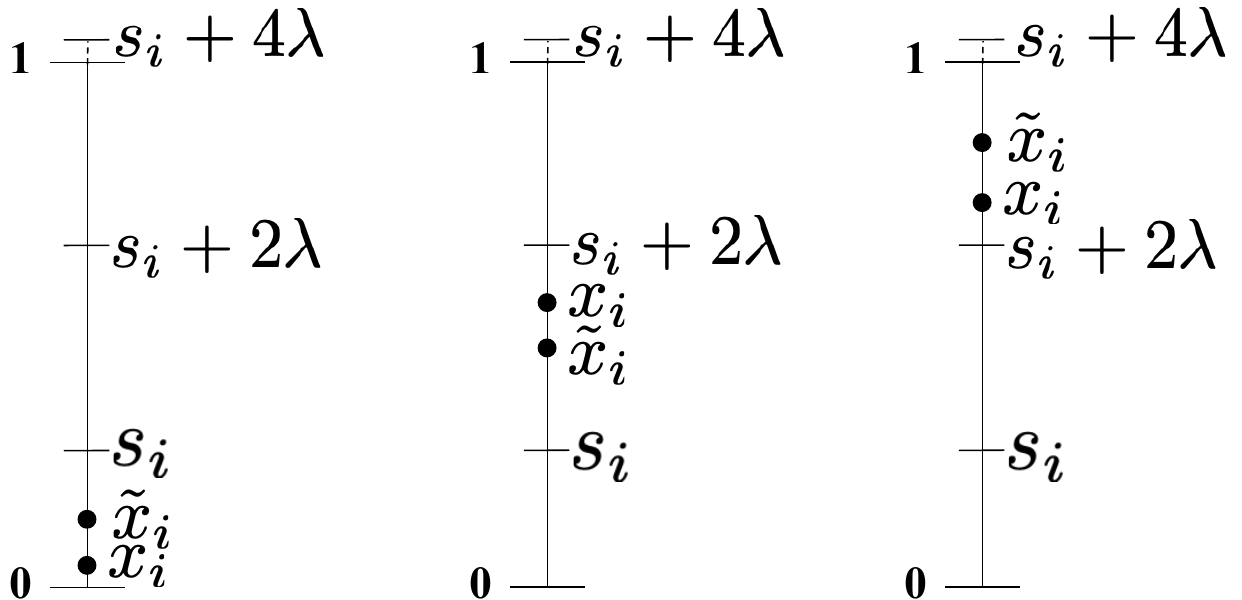}
    \caption{Example of $\tilde{x}_i$ in the $\lambda < 0.5$ case. In this case, the interval $[0,1]$ is split into sub-intervals $[0,s_i]$, $(s_i,s_i+2\lambda]$, and $(s_i+2\lambda, 1]$. $\tilde{x}_i$ is assigned to the middle of whichever of these intervals $x_i$ falls into.}
    \label{fig:split_randomization_small}
\end{figure}
In the appendix, we provide intuition for this, by demonstrating that despite having similar \textit{marginal} distributions, the \textit{joint} distributions of $\tilde{\vx}$ and $(\vx + \epsilon)$ which can be generated by SSN and additive uniform noise, respectively, are in fact quite different. An example is shown in Figure \ref{fig:l1_indep}.

\subsection{General Case, including  $\lambda < 0.5$}
In the case $\lambda < 0.5$, we split the $[0,1]$ interval not only at $s_i \in [0,2\lambda]$, but also at every value $s_i+2\lambda n$, for $n \in \sN$. An example is shown in Figure \ref{fig:split_randomization_small}. Note that this formulation covers the $\lambda \geq 0.5$ case as well (the splits for $n \geq 1$ are simply not relevant).
\setcounter{theorem}{\thetheorem-1}

\begin{theorem}[General Case]  \label{thm:main_case_2}
For any $f: \R^d   \rightarrow [0,1]$, and $\lambda> 0$ let $\vs \in [0,2\lambda]^d$ be a random variable,  with a fixed distribution such that:
\begin{equation}
      s_i \sim \gU(0,2\lambda), \,\,\,\, \forall i.
\end{equation}
Note that the components $s_1, ..., s_d$ are \textbf{not} required to be  distributed independently from each other. Then, define:
\begin{align}
\tilde{x}_i &:=
\frac{ \min(2\lambda \ceil{\frac{x_i - s_i}{2\lambda} } + s_i, 1) }{2} \\
&+ \frac{\max(2\lambda \ceil{\frac{x_i - s_i}{2\lambda} -1} + s_i, 0)}{2}\
,\,\,\,\, \forall i \\
p(\vx) &:=\mathop{\E}_{\vs}\left[ f(\tilde{\vx})\right].
\end{align}
Then, $p(.)$ is $1/(2\lambda)$-Lipschitz with respect to the $\ell_1$ norm.
\end{theorem}

The proof for this case, as well as its derandomization, are provided in the appendix. As with the $\lambda \geq 0.5$ case, the derandomization allows for $p(\vx)$ to be computed exactly using $2\lambda q$ evaluations of $f$.
\section{Experiments}
We evaluated the performance of our method on CIFAR-10 and ImageNet datasets, matching all experimental conditions from \cite{pmlr-v119-yang20c} as closely as possible (further details are given in the appendix.) Certification performance data is given in Table \ref{tab:cifar_results} for CIFAR-10 and Figure \ref{fig:imagenet} for Imagenet.
Note that instead of using the hyperparameter $\lambda$, we report experimental results in terms of $\sigma = \lambda/\sqrt{3}$: this is to match \cite{pmlr-v119-yang20c}, where this gives the standard deviation of the uniform noise. 

We find that DSSN significantly outperforms \citet{pmlr-v119-yang20c} on both datasets, particularly when certifying for large perturbation radii. For example, at $\rho=4.0$, DSSN provides a 36\% certified accuracy on CIFAR-10, while uniform additive noise provides only 27\% certified accuracy.
In addition to these numerical improvements, DSSN certificates are \textit{exact} while randomized certificates hold only with \textit{high-probability}. Following \citet{pmlr-v119-yang20c}, all certificates reported here for randomized methods hold with $99.9\%$  probability: there is no such failure rate for DSSN.

Additionally, the certification runtime of DSSN is reduced compared to \citet{pmlr-v119-yang20c}'s method. Although in contrast to \citet{pmlr-v119-yang20c}, our certification time scales linearly with the noise level, the fact that \citet{pmlr-v119-yang20c} uses 100,000 smoothing samples makes our method much faster even at the largest tested noise levels: see Figure \ref{fig:time_per_appendix}. For example, on CIFAR-10 at $\sigma=3.5$, we achieve an average runtime of 0.41 seconds per image, while \citet{pmlr-v119-yang20c}'s method requires 13.44 seconds per image.
\begin{figure}
    \centering
    \includegraphics[width=0.48\textwidth]{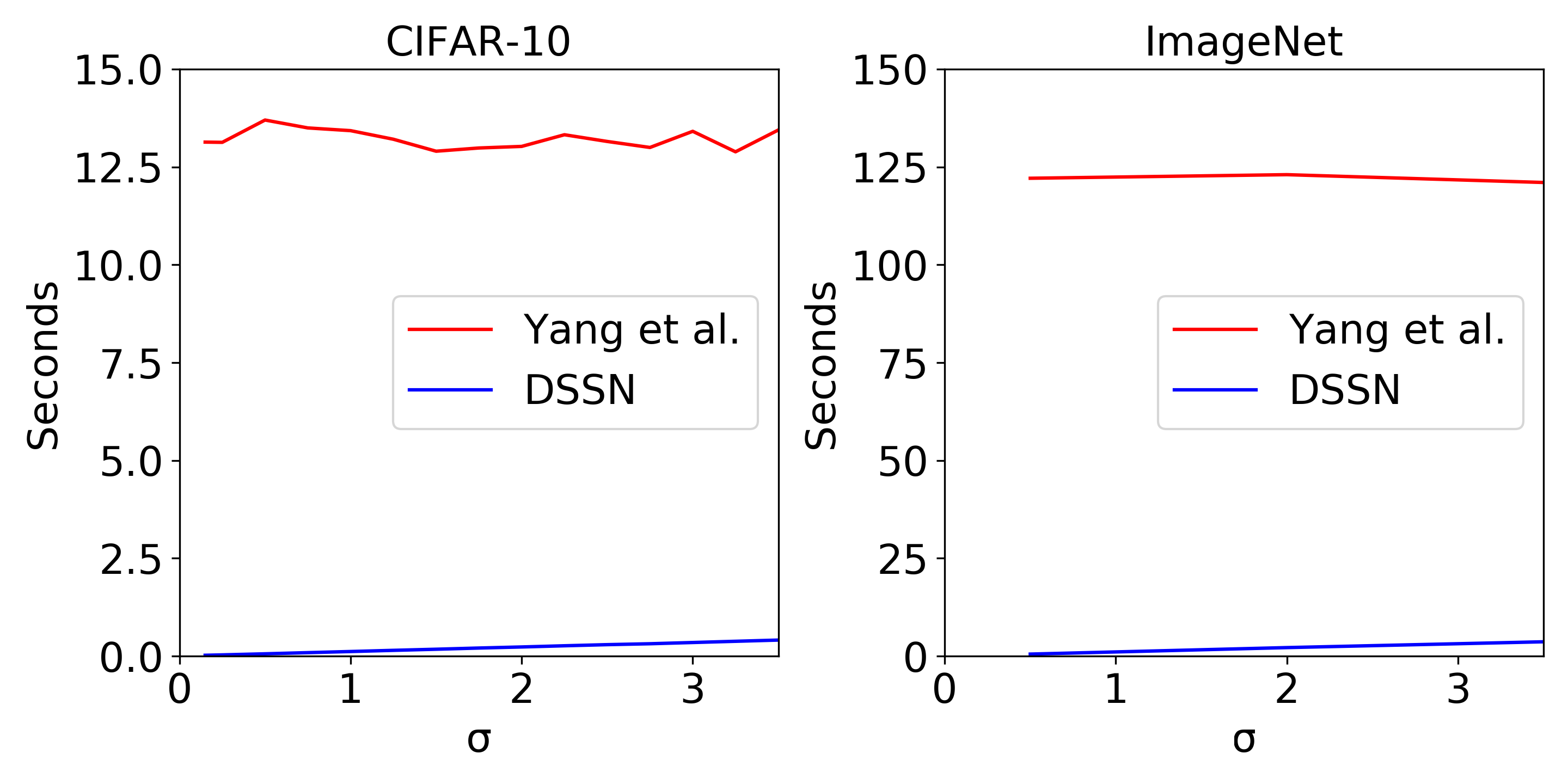}
    \caption{Comparison of the certification time per image of DSSN and \citet{pmlr-v119-yang20c}'s uniform additive noise method. We used a single NVIDIA 2080 Ti GPU. }
    \label{fig:time_per_appendix}
\end{figure}

\citet{pmlr-v119-yang20c} tests using both standard training on noisy samples as well as stability training \citep{li2019certified}: while our method dominates in both settings, we find that the stability training leads to less of an improvement in our methods, and is in some cases detrimental. For example, in Table \ref{tab:cifar_results}, the best certified accuracy is always higher under stability training for uniform additive noise, while this is not the case for DSSN at $\rho < 3.0$.  Exploring the cause of this may be an interesting direction for future work.\footnote{On CIFAR-10, \citet{pmlr-v119-yang20c} also tests using semi-supervised and transfer learning approaches which incorporate data from other datasets. We consider this beyond the scope of this work where we consider only the supervised learning setting.} 

In Figure \ref{fig:random_derandom}, we compare the uniform additive smoothing method to DSSN, as well the \textit{randomized} form of SSN with independent splitting noise. At mid-range noise levels, the primary benefit of our method is due to derandomization; while at large noise levels, the differences in noise representation discussed in Section \ref{sec:marg_distrib} become more relevant. In the appendix, we provide complete certification data at all tested noise levels, using both DSSN and SSN with independent noise, as well as more runtime data. Additionally we further explore the effect of the noise representation: given that Equation \ref{eq:noise_mapping} shows a simple mapping between (the marginal distributions of) SSN and uniform additive noise, we tested whether the gap in performance due to noise representations can be eliminated by a ``denoising layer'', as trained in \cite{salman2020denoised}. We did not find evidence of this: the gap persists even when using denoising.

\begin{table*}[]
\small
    \begin{tabular}{|c|l|l|l|l|l|l|l|l|}
\hline
&$\rho =$ 0.5&$\rho =$ 1.0&$\rho =$ 1.5&$\rho =$ 2.0&$\rho =$ 2.5&$\rho =$ 3.0&$\rho =$ 3.5&$\rho =$ 4.0\\
\hline
Uniform&70.54\%&58.43\%&50.73\%&43.16\%&33.24\%&25.98\%&20.66\%&17.12\%\\
Additive Noise&(83.97\%&(78.70\%&(73.05\%&(73.05\%&(69.56\%&(62.48\%&(53.38\%&(53.38\%\\
&@ $\sigma$=0.5)&@ $\sigma$=1.0)&@ $\sigma$=1.75)&@ $\sigma$=1.75)&@ $\sigma$=2.0)&@ $\sigma$=2.5)&@ $\sigma$=3.5)&@ $\sigma$=3.5)\\
\hline
Uniform&71.09\%&60.36\%&52.86\%&47.08\%&42.26\%&38.55\%&33.76\%&27.12\%\\
Additive Noise&(78.79\%&(74.27\%&(65.88\%&(63.32\%&(57.49\%&(57.49\%&(57.49\%&(57.49\%\\
(+Stability Training)&@ $\sigma$=0.5)&@ $\sigma$=0.75)&@ $\sigma$=1.5)&@ $\sigma$=1.75)&@ $\sigma$=2.5)&@ $\sigma$=2.5)&@ $\sigma$=2.5)&@ $\sigma$=2.5)\\
\hline
{\bf DSSN - Our Method}&\textbf{72.25\%}&\textbf{63.07\%}&\textbf{56.21\%}&\textbf{51.33\%}&\textbf{46.76\%}&42.66\%&38.26\%&33.64\%\\
&(81.50\%&(77.85\%&(71.17\%&(67.98\%&(65.40\%&(65.40\%&(65.40\%&(65.40\%\\
&@ $\sigma$=0.75)&@ $\sigma$=1.25)&@ $\sigma$=2.25)&@ $\sigma$=3.0)&@ $\sigma$=3.5)&@ $\sigma$=3.5)&@ $\sigma$=3.5)&@ $\sigma$=3.5)\\
\hline
{\bf DSSN - Our Method}&71.23\%&61.04\%&54.21\%&49.39\%&45.45\%&\textbf{42.67\%}&\textbf{39.46\%}&\textbf{36.46\%}\\
(+Stability Training)&(79.00\%&(71.29\%&(66.04\%&(64.26\%&(59.88\%&(57.16\%&(56.29\%&(54.96\%\\
&@ $\sigma$=0.5)&@ $\sigma$=1.0)&@ $\sigma$=1.5)&@ $\sigma$=1.75)&@ $\sigma$=2.5)&@ $\sigma$=3.0)&@ $\sigma$=3.25)&@ $\sigma$=3.5)\\
\hline

    \end{tabular}
    \caption{Summary of results for CIFAR-10. Matching \citet{pmlr-v119-yang20c}, we test on 15 noise levels ($\sigma \in \{0.15, 0.25n  \text{ for }1 \leq n \leq 14\}$). We report the best certified accuracy at a selection of radii $\rho$, as well as the clean accuracy and noise level of the associated classifier. Our method dominates at all radii, although stability training seems to be less useful for our method. Note that these statistics are based on reproducing \citet{pmlr-v119-yang20c}'s results; they are all within $\pm 1.5$ percentage points of \citet{pmlr-v119-yang20c}'s reported statistics.
    \label{tab:cifar_results}}
\end{table*}
\begin{figure*}
    \centering
    \includegraphics[width=0.9\textwidth, trim={0 0.3cm 0 0.3cm},clip]{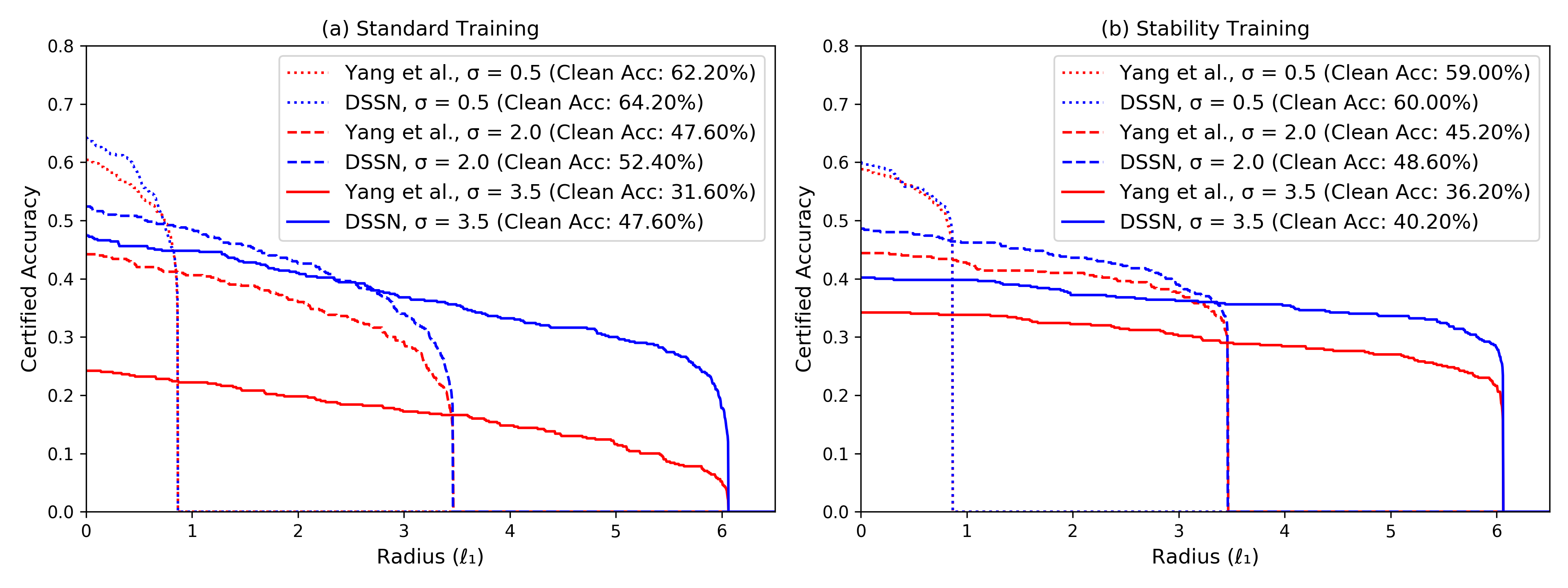}
    \caption{Results on ImageNet. We report results at three noise levels, with and without stability training. Our method dominates in all settings: however, especially at large noise, stability training seems to \textit{hurt} our clean accuracy, rather than help it. }
    \label{fig:imagenet}
\end{figure*}
\begin{figure*}
    \centering
    \includegraphics[width=.98\textwidth,trim={0 0.2cm 0 0.4cm},clip]{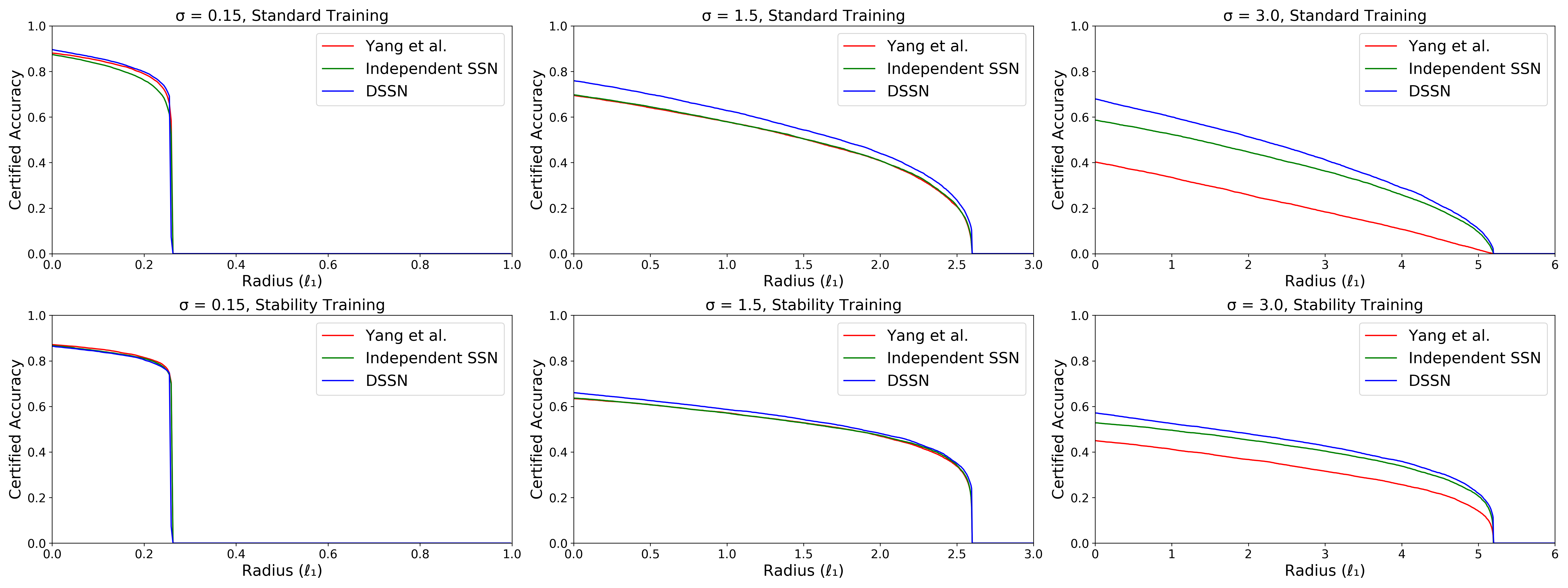}
    \caption{ Comparison on CIFAR-10 of additive smoothing \citep{pmlr-v119-yang20c} to DSSN, as well as SSN with \textit{random, independent} splitting noise, using the estimation scheme from \cite{pmlr-v119-yang20c}.
    At very small levels of noise ($\sigma = 0.15$), there is little difference between the methods: in fact, with stability training, additive smoothing slightly outperforms DSSN. At intermediate noise levels, additive noise and independent SSN perform very similarly, but DSSN outperforms both. This suggests that, at this level, the primary benefit of DSSN is to eliminate estimation error (Section \ref{sec:yang}). At high noise levels, the largest gap is between additive noise and independent SSN, suggesting that in this regime, most of the performance benefits of DSSN are due to improved base classifier performance (Section \ref{sec:marg_distrib}).}
    \label{fig:random_derandom}
\end{figure*}

\section{Conclusion}
In this work, we have improved the state-of-the-art smoothing-based robustness certificate for the $\ell_1$ threat model, and provided the first scalable, general-use derandomized ``randomized smoothing'' certificate for a norm-based adversarial threat model. To accomplish this, we proposed a novel \textit{non-additive} smoothing method. Determining whether such methods can be extended to other $\ell_p$ norms remains an open question for future work.
\section*{Acknowledgements}
 This project was supported in part by NSF CAREER AWARD 1942230, HR00111990077, HR001119S0026, HR00112090132, NIST 60NANB20D134 and Simons Fellowship on ``Foundations of Deep Learning.''
\bibliography{main}
\bibliographystyle{icml2021}
\appendix
\include{Supplement}

\end{document}

%% file: math_commands.tex
\usepackage{amsmath,amsfonts,bm}

\def\eqref#1{equation~\ref{#1}}

\def\ceil#1{\lceil #1 \rceil}
\def\floor#1{\lfloor #1 \rfloor}
\def\1{\bm{1}}
\def\0{\bm{0}}

\def\eps{{\epsilon}}

\def\vs{{\bm{s}}}

\def\vv{{\bm{v}}}

\def\vx{{\bm{x}}}

\def\vz{{\bm{z}}}

\DeclareMathAlphabet{\mathsfit}{\encodingdefault}{\sfdefault}{m}{sl}
\SetMathAlphabet{\mathsfit}{bold}{\encodingdefault}{\sfdefault}{bx}{n}

\def\gB{{\mathcal{B}}}

\def\gD{{\mathcal{D}}}

\def\gS{{\mathcal{S}}}

\def\gU{{\mathcal{U}}}
\def\gV{{\mathcal{V}}}

\def\gX{{\mathcal{X}}}

\def\sN{{\mathbb{N}}}

\newcommand{\E}{\mathbb{E}}

\newcommand{\R}{\mathbb{R}}

\newcommand{\bx}{\mathbf{x}}

\newcommand{\bv}{\mathbf{v}}

%% file: Supplement.tex
\section{Proofs}
\setcounter{theorem}{0}
\begin{theorem}[\citet{lee2019tight}] For any $f: \R^d   \rightarrow [0,1]$ and parameter $\lambda \in \R^{+}$, define:
\begin{equation}
    p(\vx) := \mathop{\E}_{\epsilon \sim \gU^d(-\lambda,\lambda)} \left[f(\vx + \epsilon)\right].
\end{equation}
Then, $p(.)$ is $1/(2\lambda)$-Lipschitz with respect to the $\ell_1$ norm.
\end{theorem}
\begin{proof}
Consider two arbitrary points $\vx, \vx'$ where $\delta:=\vx'-\vx$.  We consider two cases. 
\begin{itemize}
    \item Case 1: $\|\delta\|_1 \geq 2\lambda$:
    Then, because $f(\cdot)\in [0,1]$, and therefore $p(\cdot)\in [0,1]$, we have:
    \begin{equation}
        |p(\vx)- p(\vx')| \leq 1 \leq \frac{\|\delta\|_1}{2\lambda}  
    \end{equation}
    \item Case 2:  $\|\delta\|_1 < 2\lambda$:
    In this case, for each $i$, $|\delta_i| < 2\lambda$. Define $\gB(\vx)$ as the $\ell_\infty$ ball of radius $\lambda$ around $\vx$, and $\gU(\gB(\vx))$ as the uniform distribution on this ball (and, similarly $\gU(\cdot)$, on any other set). In other words:
    \begin{equation}
        p(\vx) = \mathop{\E}_{\vz \sim \gU(\gB(\vx))} f(\vz)  
    \end{equation}
    Then,
    \begin{equation}
    \begin{split}
      &|p(\vx)- p(\vx')|  \\&= |\mathop{\E}_{\vz \sim \gU(\gB(\vx))} f(\vz) - \mathop{\E}_{\vz \sim \gU(\gB(\vx'))} f(\vz) |\\ &=
      \Big|\Big(\Pr_{\vz \sim \gU(\gB(\vx))} \mkern-18mu \vz \in \gB(\vx)\setminus \gB(\vx')\mathop{\E}_{
      \substack{\vz \sim \gU(\gB(\vx)\\\setminus \gB(\vx'))
      }} \mkern-18mu f(\vz) \\&+
       \Pr_{\vz \sim \gU(\gB(\vx))} \mkern-18mu \vz \in \gB(\vx)\cap \gB(\vx')\mathop{\E}_{ \substack{\vz \sim \gU(\gB(\vx)\\\cap \gB(\vx'))
      }} \mkern-18mu f(\vz)\Big) \\&-
       \Big(\Pr_{\vz \sim \gU(\gB(\vx'))} \mkern-18mu \vz \in \gB(\vx')\setminus \gB(\vx)\mathop{\E}_{ \substack{\vz \sim \gU(\gB(\vx')\\\setminus \gB(\vx))
      }} \mkern-18mu f(\vz) \\&+
       \Pr_{\vz \sim \gU(\gB(\vx'))} \mkern-18mu \vz \in \gB(\vx)\cap \gB(\vx')\mathop{\E}_{ \substack{\vz \sim \gU(\gB(\vx)\\\cap \gB(\vx'))
      }} \mkern-18mu f(\vz)\Big)\Big|\\&=
       \Big|\Pr_{\vz \sim \gU(\gB(\vx))} \mkern-18mu \vz \in \gB(\vx)\setminus \gB(\vx')\mathop{\E}_{ \substack{\vz \sim \gU(\gB(\vx)\\\setminus \gB(\vx'))
      }} \mkern-18mu f(\vz) \\&-
       \Pr_{\vz \sim \gU(\gB(\vx'))} \mkern-18mu \vz \in \gB(\vx')\setminus \gB(\vx)\mathop{\E}_{ \substack{\vz \sim \gU(\gB(\vx')\\\setminus \gB(\vx))
      }} \mkern-18mu f(\vz) \Big|
    \end{split}
    \end{equation}
 Note that:
 \begin{equation}
 \begin{split}
      &\Pr_{\vz \sim \gU(\gB(\vx'))}\vz \in \gB(\vx')\setminus \gB(\vx) \\&= \Pr_{\vz \sim \gU(\gB(\vx))}\vz \in \gB(\vx)\setminus \gB(\vx')
 \end{split}
 \end{equation}
 Because both represent the probability of a uniform random variable on an $\ell_\infty$ ball of radius $\lambda$ taking a value outside of the region $\gB(\vx)\cap \gB(\vx')$ (which is entirely contained within both balls.) Then:
    \begin{equation} \label{eq:thm_1_pf_pr}
    \begin{split}
      &|p(\vx)- p(\vx')|  \\&=
    \Pr_{\vz \sim \gU(\gB(\vx))}\vz \in \gB(\vx)\setminus \gB(\vx') \\&\times \Big|\mathop{\E}_{\substack{\vz \sim \gU(\gB(\vx)\setminus\\ \gB(\vx'))}} f(\vz) -
       \mathop{\E}_{\substack{\vz \sim \gU(\gB(\vx')\\\setminus \gB(\vx))}} f(\vz) \Big| \\&\leq
       \Pr_{\vz \sim \gU(\gB(\vx))}\vz \in \gB(\vx)\setminus \gB(\vx').
      \end{split}
      \end{equation}
      Where, in the last line, we used the fact that $f(\cdot)\in [0,1]$. Let $\gV(\gS)$ represent the volume of a set $\gS$. Note that $\gB(\vx) \cap \gB(\vx')$ is a $d$-hyperrectangle, with each edge of length
      \begin{equation}
         \min(x_i,x'_i) +\lambda - (\max(x_i,x'_i)-\lambda) = 2\lambda - |\delta_i|
      \end{equation}
      Then following Equation \ref{eq:thm_1_pf_pr},
      \begin{equation}
      \begin{split}
     &|p(\vx)- p(\vx')|  \\&\leq 
     \frac{\gV(\gB(\vx)) - \gV(\gB(\vx) \cap \gB(\vx') )}{\gV(\gB(\vx))} \\&=
     1 - \frac{\mathop{\Pi}_{i=1}^d (2\lambda - |\delta_i|)}{(2\lambda)^d}
  \\&= 
      1 - \mathop{\Pi}_{i=1}^d\left(1 - \frac{|\delta_i|}{2\lambda}\right)
      \end{split}
      \end{equation}
      Note that, for $1 \leq d' \leq d$:
      \begin{equation}
      \begin{split}
          &\mathop{\Pi}_{i=1}^{d'}\left(1 - \frac{|\delta_i|}{2\lambda}\right) \\&= \mathop{\Pi}_{i=1}^{d'-1}\left(1 - \frac{|\delta_i|}{2\lambda}\right) - \frac{|\delta_{d'}|}{2\lambda} \mathop{\Pi}_{i=1}^{d'-1}\left(1 - \frac{|\delta_i|}{2\lambda}\right) \\&\geq
          \mathop{\Pi}_{i=1}^{d'-1}\left(1 - \frac{|\delta_i|}{2\lambda}\right) - \frac{|\delta_{d'}|}{2\lambda}
      \end{split}
      \end{equation}
      By induction:
      \begin{equation}
       \mathop{\Pi}_{i=1}^{d}\left(1 - \frac{|\delta_i|}{2\lambda}\right) \geq 
       1 - \sum_{i=1}^{d}\ \frac{|\delta_i|}{2\lambda}
      \end{equation}
      Therefore,
      \begin{equation}
          \begin{split}
               &|p(\vx)- p(\vx')|  \\&\leq 
                1 - \mathop{\Pi}_{i=1}^d\left(1 - \frac{|\delta_i|}{2\lambda}\right) \\&\leq 
                1- \left( 1 - \sum_{i=1}^{d}\ \frac{|\delta_i|}{2\lambda}\right) \\&=\frac{\|\delta\|_1}{2\lambda}  
          \end{split}
      \end{equation}
\end{itemize}
Thus, by the definition of Lipschitz-continuity,  $p$ is  $1/(2\lambda)$-Lipschitz with respect to the $\ell_1$ norm.
\end{proof}
\setcounter{theorem}{1}
\begin{theorem}[General Case] 
For any $f: \R^d   \rightarrow [0,1]$, and $\lambda> 0$ let $\vs \in [0,2\lambda]^d$ be a random variable,  with a fixed distribution such that:
\begin{equation}
      s_i \sim \gU(0,2\lambda), \,\,\,\, \forall i.
\end{equation}
Note that the components $s_1, ..., s_d$ are \textbf{not} required to be  distributed independently from each other. Then, define:
\begin{align}
\tilde{x}_i &:=
\frac{ \min(2\lambda \ceil{\frac{x_i - s_i}{2\lambda} } + s_i, 1) }{2} \\
&+ \frac{\max(2\lambda \ceil{\frac{x_i - s_i}{2\lambda} -1} + s_i, 0)}{2}\
,\,\,\,\, \forall i \\
p(\vx) &:=\mathop{\E}_{\vs}\left[ f(\tilde{\vx})\right].
\end{align}
Then, $p(.)$ is $1/(2\lambda)$-Lipschitz with respect to the $\ell_1$ norm.
\end{theorem}
\begin{proof}

Consider two arbitrary points $\vx, \vx'$ where $\delta:=\vx'-\vx$.  We consider two cases. 
\begin{itemize}
    \item Case 1: $\|\delta\|_1 \geq 2\lambda$:
    Then, because $f(\cdot)\in [0,1]$, and therefore $p(\cdot)\in [0,1]$, we have:
    \begin{equation}
        |p(\vx)- p(\vx')| \leq 1 \leq \frac{\|\delta\|_1}{2\lambda}  
    \end{equation}
    \item Case 2:  $\|\delta\|_1 < 2\lambda$:
    
    In this case, for each $i$, $|\delta_i| < 2\lambda$, and therefore $\ceil{\frac{x_i - s_i}{2\lambda} }$ and $\ceil{\frac{x_i' - s_i}{2\lambda} }$ differ by at most one. Furthermore, $\ceil{\frac{x_i - s_i}{2\lambda} }$ differs from $\ceil{\frac{x_i}{2\lambda} }$ by at most one, and similarly for $x_i'$. Without loss of generality, assume $x_i < x_i'$ (i.e., $\delta_i = |\delta_i| = x_i' - x_i$).
    
    There are two cases:
    \begin{itemize}
        \item Case A: $\ceil{\frac{x_i}{2\lambda}} = \ceil{\frac{x_i'}{2\lambda}}$. Let this integer be $n$. Then:
        \begin{itemize}
            \item $\ceil{\frac{x_i-s_i}{2\lambda}} = \ceil{\frac{x_i'-s_i}{2\lambda}} = n$ iff $\frac{s_i}{2\lambda} < \frac{x_i}{2\lambda} -(n-1)$ (which also implies $\frac{s_i}{2\lambda} < \frac{x'_i}{2\lambda} -(n-1)$).
            \item  $\ceil{\frac{x_i-s_i}{2\lambda}} = \ceil{\frac{x_i'-s_i}{2\lambda}} = n-1$ iff $\frac{s_i}{2\lambda} \geq \frac{x'_i}{2\lambda} -(n-1)$ (which also implies $\frac{s_i}{2\lambda} \geq \frac{x_i}{2\lambda} -(n-1)$).
        \end{itemize}
 Then  $\ceil{\frac{x_i - s_i}{2\lambda} }$ and $\ceil{\frac{x_i' - s_i}{2\lambda} }$ differ only if $\frac{x_i}{2\lambda} -(n-1) \leq \frac{s_i}{2\lambda} < \frac{x_i'}{2\lambda} -(n-1)$, which occurs with probability $\frac{\delta_i}{2\lambda}$.
        
        \item Case B: $\ceil{\frac{x_i}{2\lambda}} + 1 = \ceil{\frac{x_i'}{2\lambda}}$. Let $n := \ceil{\frac{x_i}{2\lambda}}$. Then $\ceil{\frac{x_i-s_i}{2\lambda}}$ and  $\ceil{\frac{x_i'-s_i}{2\lambda}}$  can differ if either:
        \begin{itemize}
            \item $\ceil{\frac{x_i-s_i}{2\lambda}} = n$ and $\ceil{\frac{x_i'-s_i}{2\lambda}} = n+1$. This occurs iff $\frac{s_i}{2\lambda} < \frac{x'_i}{2\lambda} -n$ (which also implies $\frac{s_i}{2\lambda} < \frac{x_i}{2\lambda} -(n-1)$).
            \item $\ceil{\frac{x_i-s_i}{2\lambda}} = n-1$ and $\ceil{\frac{x_i'-s_i}{2\lambda}} = n$. This occurs iff $\frac{s_i}{2\lambda} \geq  \frac{x_i}{2\lambda} -(n-1)$ (which also implies $\frac{s_i}{2\lambda} \geq \frac{x'_i}{2\lambda} -n$).
        \end{itemize}
        In other words, $\ceil{\frac{x_i-s_i}{2\lambda}} = \ceil{\frac{x_i'-s_i}{2\lambda}}$ iff:
        \begin{equation*}
            \frac{x_i}{2\lambda} -(n-1) > \frac{s_i}{2\lambda} \geq \frac{x_i'}{2\lambda} -n
        \end{equation*}
        Or equivalently:
        \begin{equation*}
            \frac{x_i}{2\lambda} -n +1 > \frac{s_i}{2\lambda} \geq \frac{x_i}{2\lambda} -n + \frac{\delta_i}{2\lambda}
        \end{equation*}
        This happens with probability $1-\frac{\delta_i}{2\lambda}$. Therefore, $\ceil{\frac{x_i-s_i}{2\lambda}}$ and  $\ceil{\frac{x_i'-s_i}{2\lambda}}$  differ with probability $\frac{\delta_i}{2\lambda}$.
    \end{itemize}
    Note that $\ceil{\frac{x_i-s_i}{2\lambda} -1}$ and  $\ceil{\frac{x_i'-s_i}{2\lambda}-1}$ differ only when $\ceil{\frac{x_i-s_i}{2\lambda}}$ and  $\ceil{\frac{x_i'-s_i}{2\lambda}}$ differ. Therefore in both cases, $\tilde{x}_i$ and  $\tilde{x}_i'$ differ with probability at most $\frac{|\delta_i|}{2\lambda}$. The rest of the proof proceeds as in the $\lambda \geq 0.5$ case in the main text.
    \end{itemize}
\end{proof}
\setcounter{corollary}{0}
\begin{corollary}[General Case] 
For any $f: \R^d   \rightarrow [0,1]$, and $\lambda \geq 0$ (with $2\lambda$ a multiple of $1/q$), let $\vs \in \left[0,2\lambda-1/q\right]_{(q)}^d + \mathbb{1}/(2q)$ be a random variable with a fixed distribution such that:
\begin{equation}
      s_i \sim \gU_{(q)}\left(0,2\lambda-1/q\right) + 1/(2q), \,\,\,\, \forall i.
\end{equation}
Note that the components $s_1, ..., s_d$ are \textbf{not} required to be  distributed independently from each other. Then, define:
\begin{align}
\tilde{\bx}_i &:=
\frac{ \min(2\lambda \ceil{\frac{\bx_i - s_i}{2\lambda} } + s_i, 1) }{2} \\
&+ \frac{\max(2\lambda \ceil{\frac{\bx_i - s_i}{2\lambda} -1} + s_i, 0)}{2}\
,\,\,\,\, \forall i \\
p(\bx) &:=\mathop{\E}_{\vs}\left[ f(\tilde{\bx})\right].
\end{align}
Then, $p(.)$ is $1/(2\lambda)$-Lipschitz with respect to the $\ell_1$ norm on the quantized domain $\bx \in [0,1]^d_{(q)}$.
\end{corollary}
\begin{proof}
The proof is substantially similar to the proof of the continuous case above. Minor differences occur in Cases 2.A and 2.B (mostly due to inequalities becoming strict, because possible values of $s_i$ are offset from values of $\bx_i$) which we show here:
    \begin{itemize}
        \item Case A: $\ceil{\frac{\bx_i}{2\lambda}} = \ceil{\frac{\bx_i'}{2\lambda}}$. Let this integer be $n$. Then:
        \begin{itemize}
            \item $\ceil{\frac{\bx_i-s_i}{2\lambda}} = \ceil{\frac{\bx_i'-s_i}{2\lambda}} = n$ iff $\frac{s_i}{2\lambda} < \frac{\bx_i}{2\lambda} -(n-1)$ (which also implies $\frac{s_i}{2\lambda} < \frac{\bx'_i}{2\lambda} -(n-1)$).
            \item  $\ceil{\frac{\bx_i-s_i}{2\lambda}} = \ceil{\frac{\bx_i'-s_i}{2\lambda}} = n-1$ iff $\frac{s_i}{2\lambda} > \frac{\bx'_i}{2\lambda} -(n-1)$ (which also implies $\frac{s_i}{2\lambda} > \frac{\bx_i}{2\lambda} -(n-1)$).
        \end{itemize}
 Then  $\ceil{\frac{\bx_i - s_i}{2\lambda} }$ and $\ceil{\frac{\bx_i' - s_i}{2\lambda} }$ differ only if $\frac{\bx_i}{2\lambda} -(n-1) < \frac{s_i}{2\lambda} < \frac{\bx_i'}{2\lambda} -(n-1)$.
 There are exactly $q\cdot \delta_i$ discrete values that $s_i$ can take such that this condition holds. This is out of $2\lambda q$ possible values over which $s_i$ is uniformly distributed. Therefore, the condition holds with probability $\frac{\delta_i}{2\lambda}$.
        
        \item Case B: $\ceil{\frac{\bx_i}{2\lambda}} + 1 = \ceil{\frac{\bx_i'}{2\lambda}}$. Let $n := \ceil{\frac{\bx_i}{2\lambda}}$. Then $\ceil{\frac{\bx_i-s_i}{2\lambda}}$ and  $\ceil{\frac{\bx_i'-s_i}{2\lambda}}$  can differ if either:
        \begin{itemize}
            \item $\ceil{\frac{\bx_i-s_i}{2\lambda}} = n$ and $\ceil{\frac{\bx_i'-s_i}{2\lambda}} = n+1$. This occurs iff $\frac{s_i}{2\lambda} < \frac{\bx'_i}{2\lambda} -n$ (which also implies $\frac{s_i}{2\lambda} < \frac{\bx_i}{2\lambda} -(n-1)$).
            \item $\ceil{\frac{\bx_i-s_i}{2\lambda}} = n-1$ and $\ceil{\frac{\bx_i'-s_i}{2\lambda}} = n$. This occurs iff $\frac{s_i}{2\lambda} >  \frac{\bx_i}{2\lambda} -(n-1)$ (which also implies $\frac{s_i}{2\lambda} > \frac{\bx'_i}{2\lambda} -n$).
        \end{itemize}
        In other words, $\ceil{\frac{\bx_i-s_i}{2\lambda}} = \ceil{\frac{\bx_i'-s_i}{2\lambda}}$ iff:
        \begin{equation*}
            \frac{\bx_i}{2\lambda} -(n-1) > \frac{s_i}{2\lambda} > \frac{\bx_i'}{2\lambda} -n
        \end{equation*}
        Or equivalently:
        \begin{equation*}
            \frac{\bx_i}{2\lambda} -n +1 > \frac{s_i}{2\lambda} > \frac{\bx_i}{2\lambda} -n + \frac{\delta_i}{2\lambda}
        \end{equation*}
         There are exactly $q\cdot (1-\delta_i)$ discrete values that $s_i$ can take such that this condition holds. This is out of $2\lambda q$ possible values over which $s_i$ is uniformly distributed. Therefore, the condition holds with probability  $\frac{1-\delta_i}{2\lambda}$. Thus, $\ceil{\frac{\bx_i-s_i}{2\lambda}}$ and  $\ceil{\frac{\bx_i'-s_i}{2\lambda}}$ differ with probability $\frac{\delta_i}{2\lambda}$.
    \end{itemize}
\end{proof}
\section{Experimental Details} \label{sec:exp_details}
For uniform additive noise, we reproduced \citet{pmlr-v119-yang20c}'s results directly, using their released code. Note that we also reproduced the training of all models, rather than using released models. For Independent SSN and DSSN, we followed the same training procedure as in \citet{pmlr-v119-yang20c}, but instead used the noise distribution of our methods during training. For DSSN, we used the same vector $\vv$ to generate noise during training and test time: note that our certificate requires  $\vv$ to be the same fixed vector whenever the classifier is used. In particular, we used a pseudorandom array generated using the Mersenne Twister algorithm with seed 0, as implemented in NumPy as numpy.random.RandomState. This is guaranteed to produce identical results on all platforms and for all future versions of NumPy, given the same seed, so in practice we only store the seed (0). In Section \ref{sec:seed}, we explore the sensitivity of our method to different choices of pseudorandom seeds.

In a slight deviation from \citet{pmlr-v97-cohen19c}, \citet{pmlr-v119-yang20c} uses different noise vectors for each sample in a batch when training (\citet{pmlr-v97-cohen19c} uses the same $\epsilon$ for all samples in a training batch to improve speed). We follow \citet{pmlr-v119-yang20c}'s method: this means that when training DSSN, we train the classifier on each sample only once per epoch, with a single, randomly-chosen value of $s_\text{base}$, which varies between samples in a batch.

Training parameters (taken from \citet{pmlr-v119-yang20c}) were as follows (Table \ref{tab:training_params}):
\begin{table}[h]
    \centering
    \begin{tabular}{|c|c|c|}
    \hline
        & CIFAR-10& ImageNet\\
        \hline
       Architecture& WideResNet-40& ResNet-50  \\
               \hline
        Number of Epochs& 120&  30\\
                \hline
        Batch Size& 64 \footnotemark& 64  \\
                \hline
        Initial &0.1& 0.1  \\
         Learning Rate&&  \\
                \hline
        LR Scheduler &Cosine &Cosine   \\
         & Annealing& Annealing  \\
    \hline
    \end{tabular}
    \caption{Training parameters for experiments.}
    \label{tab:training_params}
\end{table}
\footnotetext{There is a discrepancy between the code and the text of \citet{pmlr-v119-yang20c} about the batch size used for training on CIFAR-10: the paper says to use a batch size of 128, while the instructions for reproducing the paper's results released with the code use a batch size of 64. Additionally, inspection of one of \citet{pmlr-v119-yang20c}'s released models indicates that a batch size of 64 was in fact used. (In particular, the ``num\_batches\_tracked'' field in the saved model, which counts the total number of batches used in training, corresponded with a batch size of 64.) We therefore used a batch size of 64 in our reproduction, assuming that the discrepancy was a result of a typo in that paper.} 

For all training and certification results in the main text, we used a single NVIDIA 2080 Ti GPU. (Some experiments with denoisers, in Section \ref{sec:denoise}, used two GPUs.)

For testing, we used the entire CIFAR-10 test set (10,000 images) and a subset of 500 images of ImageNet (the same subset used by \citet{pmlr-v97-cohen19c}). 

When reporting clean accuracies for randomized techniques (uniform additive noise and Independent SSN), we followed \cite{pmlr-v119-yang20c} by simply reporting the percent of samples for which the $N_0 = 64$ initial noise perturbations, used to pick the top class during certification, actually selected the correct class. (Notably, \cite{pmlr-v119-yang20c} does not use an ``abstain'' option for prediction, as some other randomized smoothing works \cite{pmlr-v97-cohen19c} do.) On the one hand, this is an inexact estimate of the accuracy of the \textit{true} classifier $p(\vx)$, which uses the true expectation. On the other hand, it is the actual, empirical accuracy of a classifier that is being used in practice. This is not an issue when reporting the clean accuracy for DSSN, which is exact.

In DSSN, following \citet{DBLP:conf/nips/0001F20a} (discussed in Section \ref{sec:prior_derandomized}), if two classes tie in the number of ``votes'', we predict the first class lexicographically: this means that we can certify robustness up to \textit{and including} the radius $\rho$, because we are guaranteed consistent behavior in the case of ties. Reported certified radii for DSSN should therefore be interpreted to guarantee robustness even in the $\|\bx-\bx'\|_1 = \rho$ case. (This is not a meaningful distinction in randomized methods where the space is taken as continuous).

\begin{table*}[ht]
\begin{tabular}{|c|l|l|l|l|l|l|l|l|} 
\hline
&$\rho =$ 0.5&$\rho =$ 1.0&$\rho =$ 1.5&$\rho =$ 2.0&$\rho =$ 2.5&$\rho =$ 3.0&$\rho =$ 3.5&$\rho =$ 4.0\\
\hline
Seed = 0&72.25\%&63.07\%&56.21\%&51.33\%&46.76\%&42.66\%&38.26\%&33.64\%\\
&(81.50\%&(77.85\%&(71.17\%&(67.98\%&(65.40\%&(65.40\%&(65.40\%&(65.40\%\\
&@ $\sigma$=0.75)&@ $\sigma$=1.25)&@ $\sigma$=2.25)&@ $\sigma$=3.0)&@ $\sigma$=3.5)&@ $\sigma$=3.5)&@ $\sigma$=3.5)&@ $\sigma$=3.5)\\
\hline
Seed = 1&72.01\%&62.73\%&56.03\%&51.20\%&46.71\%&42.45\%&37.87\%&33.08\%\\
&(81.85\%&(75.64\%&(72.19\%&(67.65\%&(66.93\%&(66.19\%&(66.19\%&(66.19\%\\
&@ $\sigma$=0.75)&@ $\sigma$=1.5)&@ $\sigma$=2.0)&@ $\sigma$=3.0)&@ $\sigma$=3.25)&@ $\sigma$=3.5)&@ $\sigma$=3.5)&@ $\sigma$=3.5)\\
\hline
Seed = 2&72.62\%&62.79\%&56.06\%&51.02\%&46.85\%&42.52\%&38.22\%&33.53\%\\
&(81.19\%&(74.26\%&(70.13\%&(70.13\%&(65.33\%&(65.33\%&(65.33\%&(65.33\%\\
&@ $\sigma$=0.75)&@ $\sigma$=1.75)&@ $\sigma$=2.5)&@ $\sigma$=2.5)&@ $\sigma$=3.5)&@ $\sigma$=3.5)&@ $\sigma$=3.5)&@ $\sigma$=3.5)\\
\hline
\end{tabular}
\caption{Comparison of DSSN using different random seeds to generate $\vv$  on CIFAR-10. Matching \citet{pmlr-v119-yang20c}, we test on 15 noise levels ($\sigma \in \{0.15, 0.25n  \text{ for }1 \leq n \leq 14\}$). We report the best certified accuracy at a selection of radii $\rho$, as well as the clean accuracy and noise level of the associated classifier. We find very little difference between the different seed values, with all certified accuracies within $\pm 0.65$ percentage points of each other.}
\label{tab:seed}
\end{table*}
\section{Effect of pseudorandom choice of $\vv$} \label{sec:seed}

In Section \ref{sec:exp_details}, we mention that the vector $\vv$ used in the derandomization of DSSN, which must be re-used every time the classifier is used, is generated pseudorandomly, using a seed of 0 in all experiments. In this section, we explore the sensitivity of our results to the choice of vector $\vv$, and in particular to the choice of random seed. To do this, we repeated all standard-training DSSN experiments on CIFAR-10, using two additional choices of random seeds. We performed both training and certification using the assigned $\vv$ vector for each experiment. Result are summarized in Table \ref{tab:seed}. We report a tabular summary, rather than certification curves, because the curves are too similar to distinguish. In general, the choice of random seed to select $\vv$ does not seem to impact the certified accuracies: all best certified accuracies were within $0.65$ percentage points of each other. This suggests that our method is robust to the choice of this hyperparameter.
\section{Effect of a Denoiser} \label{sec:denoise}

As shown in Figure \ref{fig:random_derandom} in the main text, at large $\lambda$, there is a substantial benefit to SSN which is unrelated to derandomization, due to the differences in noise distributions discussed in Section 4.2.1. However, Equation \ref{eq:noise_mapping} shows that the difference between uniform additive noise and Independent SSN is a simple, deterministic transformation on each pixel. We therefore wondered whether training a denoiser network, to learn the relationship between $\vx$ and the noisy sample ($\vx + \epsilon$ or $\tilde{{\vx}}$), would eliminate the differences between the methods. 
\citet{salman2020denoised} proposes methods of training denoisers for randomized smoothing, in the context of using smoothing on pre-trained classifiers. In this context, the noisy image first passes through a denoiser network, before being passed into a classification network trained on clean images.  We used their code (and all default parameters), in three variations:
\begin{enumerate}
    \item \textbf{Stability Denoising}:  In this method, the pre-trained classifier network is required for training the denoiser. The loss when training the denoiser is based on the consistency between the logit outputs of the classifier on the clean input $\vx$ and on the denoised version of the noisy input. This is the best-performing method in \cite{salman2020denoised}. However, note that it does not directly use the pixel values of $\vx$ when training the denoiser, and therefore might not ``learn'' the correspondence between clean and noisy samples (Figure \ref{fig:compare_representations} in the main text) as easily.
    \item \textbf{MSE Denoising}: This trains the denoiser via direct supervised training, with the objective of reducing the mean squared error difference between the pixel values of the clean and denoised samples. Then, classification is done using a classifier that is pre-trained only on clean samples. This performs relatively poorly in \cite{salman2020denoised}, but should directly learn the correspondence between clean and noisy samples.
     \item \textbf{MSE Denoising with Retraining}: For this experiment, we trained an MSE denoiser as above, but \textit{then} trained the entire classification pipeline (the denoiser $+$ the classifier) on noisy samples. Note that the classifier is trained from scratch in this case, with the pre-trained denoiser already in place (but being fine-tuned as the classifier is trained).
\end{enumerate}
We tested on CIFAR-10, at three different noise levels, without stability training. See Figure \ref{fig:denoiser} for results. Overall, we find that at high noise, there is still a significant gap in performance between Independent SSN and \cite{pmlr-v119-yang20c}'s method, using all of the denoising techniques. One possible explanation is that it is also more difficult \textit{for the denoiser} to learn the noise distribution of  \cite{pmlr-v119-yang20c}, compared to our distributions. 

\section{Additive and splitting noise allow for different types of joint noise distributions}
In Section 4.2 in the main text, we showed that, in the $\lambda=0.5$ case, SSN leads to marginal distributions which are simple affine transformations of the marginal distributions of the uniform additive smoothing noise (Equation \ref{eq:equivalence_lambda_half}).  However, we also showed (Proposition  \ref{prop:uniform_broken}) that, even in this case, certification is not possible using arbitrary joint distributions of $\epsilon$ with uniform additive noise, as it is with SSN. This difference is explained by the fact that, even for $\lambda = 0.5$, the joint distributions of $(\vx + \epsilon)$ which can be generated by uniform additive noise and the joint distributions of $\tilde{\vx}$ which can be generated by SSN respectively are in fact quite different.

To quantify this, consider a pair of two joint distributions: $\gD$, with marginals uniform on $[-0.5, 0.5]$, and $\gS$, with marginals uniform on $[0, 1]$. Let $\gD$ and $\gS$ be considered \textit{equivalent}  if, for $\epsilon \sim \gD$ and $\vs \sim \gS$:
\begin{equation}
 \tilde{\vx}  \sim (1/2) (\vx + \epsilon) + \mathbb{1}/4 \,\,\,\,\,\,\forall \vx
\end{equation}
where $\tilde{\vx}$ is generated using the SSN noise $\vs$ (compare to Equation \ref{eq:equivalence_lambda_half} in the main text).
\begin{proposition}
The only pair of equivalent joint distributions $(\gD,\gS)$ is $\gD \sim \gU^d(-0.5,0.5)$,  $\gS \sim \gU^d(0,1)$.
\end{proposition}
\begin{proof}
We first describe a special property of SSN (with $\lambda$ = 0.5):

Fix a smoothed value $\tilde{\vx}'$, and let $\gX'$ be the set of all inputs $\vx$ such that $\tilde{\vx}'$ can be generated from $\vx$ under \textit{any} joint splitting distribution $\gS$.
From Figure \ref{fig:compare_representations}-a in the main text, we can see that this is simply
\begin{equation} \label{eq:apdx_ball}
\gX' = \{\vx | \tilde{x}'_i \leq x_i/2 + (1/2)  \leq \tilde{x}'_i  + (1/2)  \,\,\,\, \forall i\}. 
\end{equation}

Notice that to generate $\tilde{\vx}'$, \textit{regardless of the value of $\vx \in \gX'$}, the splitting vector $\vs$ must be exactly the following:

\begin{equation}
    s_i = \begin{cases}
   2\tilde{x}_i' &\text{     if  }  \tilde{x}_i' < 1/2\\
    2\tilde{x}_i' -1 &\text{     if  }  \tilde{x}_i' \geq 1/2\\
    \end{cases}
\end{equation}
(This is made clear by Figure  \ref{fig:split_randomization_1} in the main text.)

If $\vx \in \gX'$, then $\tilde{\vx}'$ will be generated iff this value of $\vs$ is chosen. Therefore, given a fixed splitting distribution $\gS$, the probability of generating $\tilde{\vx}'$ must be \textit{constant} for all points in $\gX'$.

Now, we compare to uniform additive noise. In order for $\gD$ and $\gS$ to be equivalent, for the fixed noised point $(\vx+\epsilon)' = 2\tilde{\vx'} - \mathbb{1}/2$, it must be the case that all points in $\gX'$ are equally likely to generate $(\vx+\epsilon)'$. But note from Equation \ref{eq:apdx_ball} that $\gX'$ is simply the uniform $\ell_\infty$ ball of radius 0.5 around $(\vx+\epsilon)'$. This implies that $\gD$ must be the uniform distribution $\gD \sim \gU^d(-0.5,0.5)$, which is equivalent to the splitting distribution $\gS \sim \gU^d(0,1)$.
\end{proof}
The \textit{only} case when SSN and uniform additive noise can produce similar distributions of noisy samples is when all noise components are independent. This helps us understand how SSN can work with \textit{any} joint distribution of splitting noise, while uniform additive noise has only been shown to produce accurate certificates when all components of $\epsilon$ are independent. 
\section{Tightness of Theorem \ref{thm:main_case_2}}
Here, we discuss the tightness of our certification result. Theorem \ref{thm:main_case_2} is tight in the following sense:
\begin{proposition} \label{thm:prop_tighnness_1}
For any $\lambda> 0$ and a random variable $\vs \in [0,2\lambda]^d$,  with any fixed distribution such that:
\begin{equation}
      s_i \sim \gU(0,2\lambda), \,\,\,\, \forall i,
\end{equation}
there exists a $f: \R^d   \rightarrow [0,1]$, such that if we define:
\begin{align}
\tilde{x}_i &:=
\frac{ \min(2\lambda \ceil{\frac{x_i - s_i}{2\lambda} } + s_i, 1) }{2} \\
&+ \frac{\max(2\lambda \ceil{\frac{x_i - s_i}{2\lambda} -1} + s_i, 0)}{2}\
,\,\,\,\, \forall i \\
p(\vx) &:=\mathop{\E}_{\vs}\left[ f(\tilde{\vx})\right].
\end{align}
then, $p(.)$ is not $c$-Lipschitz with respect to the $\ell_1$ norm for any $c < 1/(2\lambda)$.
\end{proposition}
In other words, we cannot make the Lipschitz constraint any tighter without some base classifier $f$ providing a counterexample. Note that this result holds for any legal choice of joint distribution of $\vs$.
\begin{proof}
We consider two cases, on $\lambda$:
\begin{itemize}
    \item \textbf{Case 1: $\lambda < 0.5$:} Consider the following base classifier:
    \begin{equation}
        f(\tilde{\vx}) = \begin{cases}
        1& \text{     if  } \tilde{x}_1 > \lambda\\\
        0& \text{     otherwise.}\end{cases}
    \end{equation}
Now, consider the points $\vx = [0,0,0,0,...]$ and $\vx' = [2\lambda,0,0,0,...]$. Note that $\|\vx' - \vx\|_1 = 2\lambda$. From the definition of $\tilde{\bv}$, we have (with probability 1):
\begin{equation*}
\begin{split}
        \tilde{x}_1 &= \frac{s_1}{2}\\
    \tilde{x}_1' &= \min\Big(\lambda +\frac{s_1}{2},\frac{1}{2}\Big) + \frac{s_1}{2}
\end{split}
\end{equation*}
Note that this means that, with probability 1, we have $\tilde{x}_1 \leq  \lambda$, and therefore $ f(\tilde{\vx}) = 0$. Similarly, with probability 1, $\tilde{x}'_1 > \lambda$, so  $f(\tilde{\vx}') = 1$. Then, for all $c < 1/(2\lambda)$:
\begin{equation*}
    \begin{split}
        |p(\vx) - p(\vx')| &= \\
          |\mathop{\E}_{\vs}\left[ f(\tilde{\vx})\right] - \mathop{\E}_{\vs}\left[ f(\tilde{\vx}')\right]| &= \\
          | 0 -1 | &= 1 > c \cdot 2\lambda =  c\|\vx' - \vx\|_1
    \end{split}
\end{equation*}
    So $p(\cdot)$ is not $c$-Lipschitz.
 \item \textbf{Case 2: $\lambda \geq 0.5$:} Consider the following base classifier:
    \begin{equation}
        f(\tilde{\vx}) = \begin{cases}
        1& \text{     if  } \tilde{x}_1 > 0.5\\\
        0& \text{     otherwise.}\end{cases}
    \end{equation}
Now, consider the points $\vx = [0,0,0,0,...]$ and $\vx' = [1,0,0,0,...]$. Note that $\|\vx' - \vx\|_1 = 1$. From the definition of $\tilde{\bx}$, we have (with probability 1):
\begin{equation*}
\begin{split}
        \tilde{x}_1 &= \frac{\min(s_1,1)}{2}\\
    \tilde{x}_1' &= \frac{\min{(s_i,1)} + \1_{s_i< 1}}{2}
\end{split}
\end{equation*}
Note that this means that, with probability 1, we have $\tilde{x}_1 \leq 0.5$, and therefore $ f(\tilde{\vx}) = 0$. On the other hand, $\tilde{x}'_1 > 0.5$ iff $s \in (0,1)$ which occurs with probability $1/(2\lambda)$. Therefore, for all $c < 1/(2\lambda)$:
\begin{equation*}
    \begin{split}
        |p(\vx) - p(\vx')| &= \\
          |\mathop{\E}_{\vs}\left[ f(\tilde{\vx})\right] - \mathop{\E}_{\vs}\left[ f(\tilde{\vx}')\right]| &= \\
          | 0 -1/(2\lambda) | &= 1/(2\lambda) > c =  c\|\vx' - \vx\|_1
    \end{split}
\end{equation*}
    So $p(\cdot)$ is not $c$-Lipschitz.
\end{itemize}
\end{proof}
However, the tightness of the global Lipschitz bound on $p$ does not imply that the final \textit{certificate} result, on the minimum possible distance from $\vx$ to the decision boundary given $p(\vx)$, is necessarily a tight bound.

For simplicity, consider a binary classifier, so that the decision boundary is at $p(\vx) = 0.5$.  The certificate given by our method can be formalized as a function 
\begin{equation}
    \text{cert}(z) := 2\lambda(z-0.5),
\end{equation} which maps the value of $p(\vx)$ to the certified lower bound on the distance to the decision boundary.

A certificate function can be  considered tight if, for all $z \in (0.5,1.0]$ there exists an $f, \vx ,\vx'$ such that:
\begin{equation}
\begin{split}
    p(\vx) &= z \\ \|\vx-\vx'\|_1 &= \text{cert}(z)\\
    p(\vx') &= 0.5
\end{split}
\end{equation}
Note that, for example, the well-known smoothing-based $\ell_2$ robustness certificate proposed by \cite{pmlr-v97-cohen19c} is tight by the analogous definition.

It turns out that our certificate function is not necessarily tight by this definition. In particular, one can show for some valid choice of $\lambda$ and joint distribution of $\vs$ that this definition of tightness does not hold.

For example, consider the case where $s_1=s_2=...=s_d$, and $\lambda > 1$. We discuss this scenario briefly in Section \ref{sec:DSSN}\footnote{In that section, we discuss this distribution in the quantized case, but the  differences is not relevant to our argument here}; recall (Equation \ref{eq:all_s_equal}) that the smoothed classifier must take the form:

\begin{equation}
    p(\vx) = \frac{2\lambda -1}{2\lambda} f(0.5 \cdot \mathbb{1}) +\frac{1}{2\lambda}\mathop{\E}_{\vs < 1}\left[ f(\tilde{\vx})\right]  \\,\,\,\,\,\forall \vx, 
\end{equation}
which is the sum of a constant, and a function bounded in $[0,1/(2\lambda)]$. If $z > 0.5 + 1/(2\lambda)$, this implies that $\frac{2\lambda -1}{2\lambda} f(0.5 \cdot \mathbb{1}) > 0.5$, which implies that $p(\cdot ) > 0.5$ everywhere. This means that the tightness condition cannot hold for $z \in (0.5+1/(2\lambda), 1]$.
\section{Complete Certification Data on CIFAR-10 and ImageNet}
We provide complete certification results for uniform additive noise, randomized SSN with independent noise, and DSSN, at all tested noise levels on both CIFAR-10 and ImageNet,  using both standard and stability training. For CIFAR-10, see Figures \ref{fig:appendix_cifar_0}, \ref{fig:appendix_cifar_1}, \ref{fig:appendix_cifar_2}, and \ref{fig:appendix_cifar_3}. For ImageNet, see Figure \ref{fig:appendix_imagenet_0}.
\begin{figure*}
    \centering
    \includegraphics[width=\textwidth]{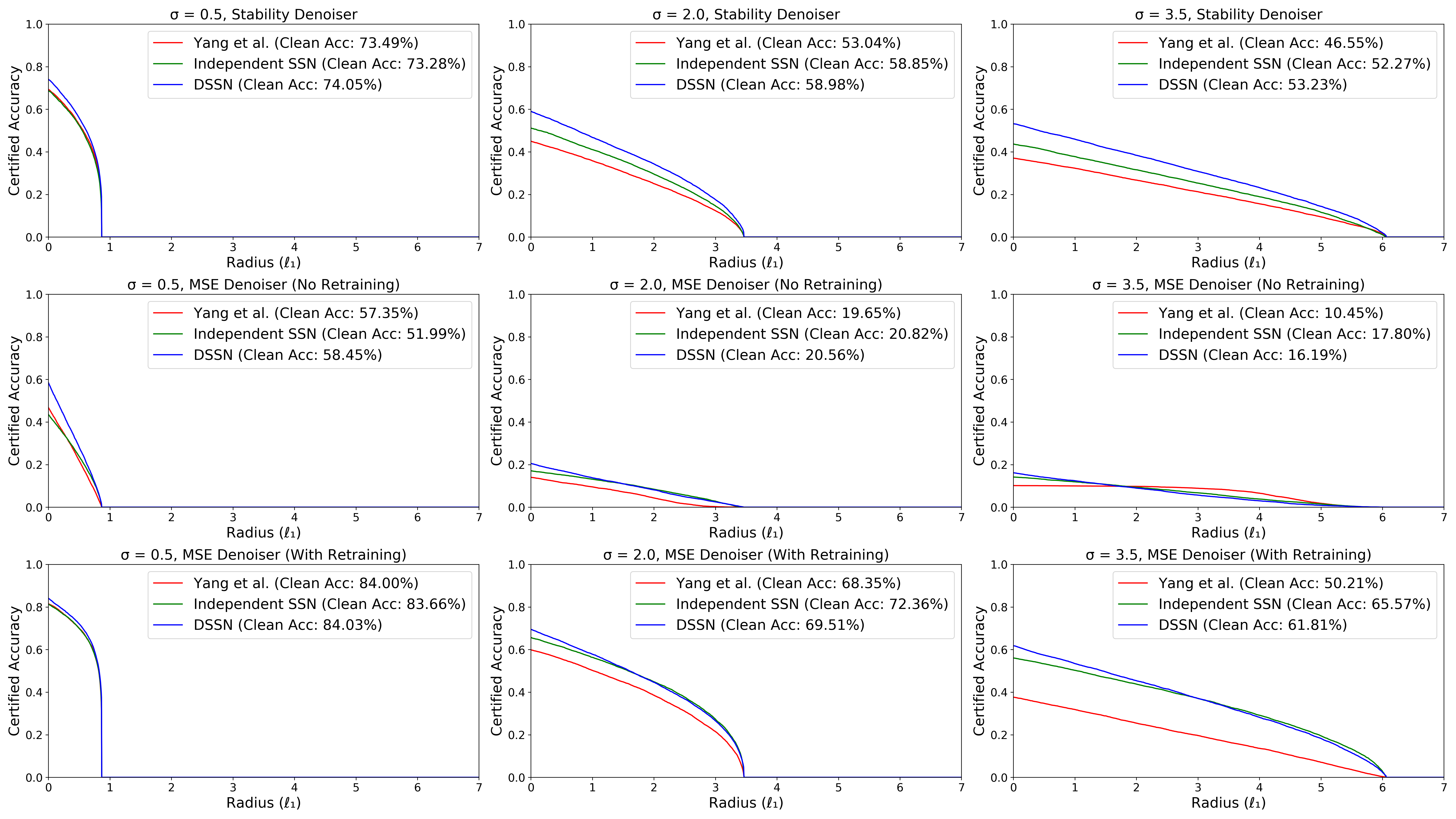}
    \caption{Certified accuracies of models trained with denoisers, for additive uniform noise, SSN with independent noise, and DSSN. See text of Section \ref{sec:denoise} for further details on the denoisers used. For $\sigma \geq 2.0$, Independent SSN outperfroms \cite{pmlr-v119-yang20c}'s method, suggesting that the difference in noise representations can not be resolved by using a denoiser. (It may appear as if \cite{pmlr-v119-yang20c}'s method is more robust at large radii for $\sigma = 3.5$ with an MSE denoiser without retraining: however, this is for a classifier with \textit{clean accuracy} $\approx 10\%$, so this is vacuous: similar results can be achieved by simply always returning the same class.)
    }
    \label{fig:denoiser}
\end{figure*}
\begin{figure*}
    \centering
    \includegraphics[width=\textwidth]{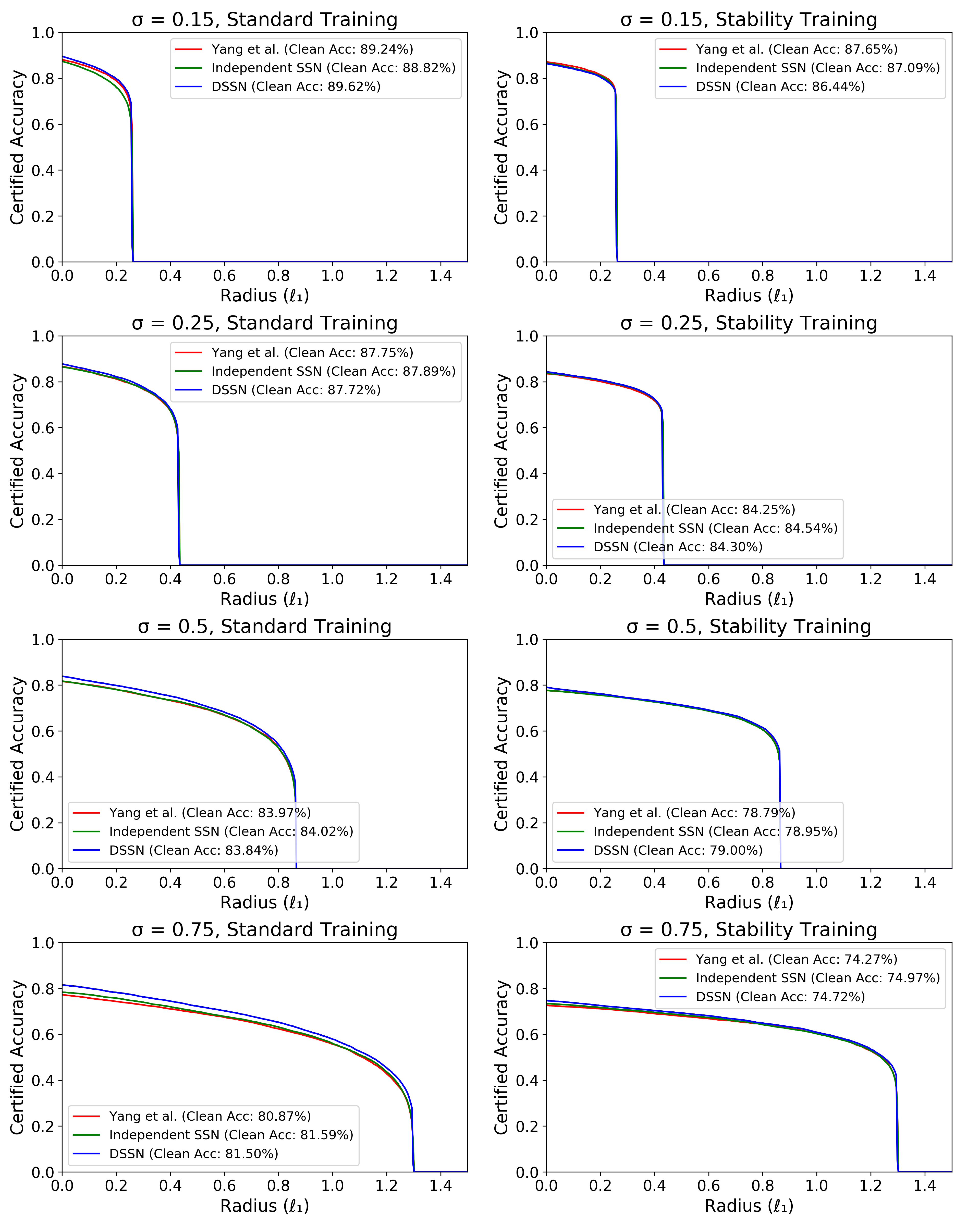}
    \caption{Certification results for CIFAR-10, comparing uniform additive noise, randomized SSN with independent noise, and DSSN, for $\sigma \in \{0.15,0.25,0.5,0.75\}$}
    \label{fig:appendix_cifar_0}.
\end{figure*}
\begin{figure*}
    \centering
    \includegraphics[width=\textwidth]{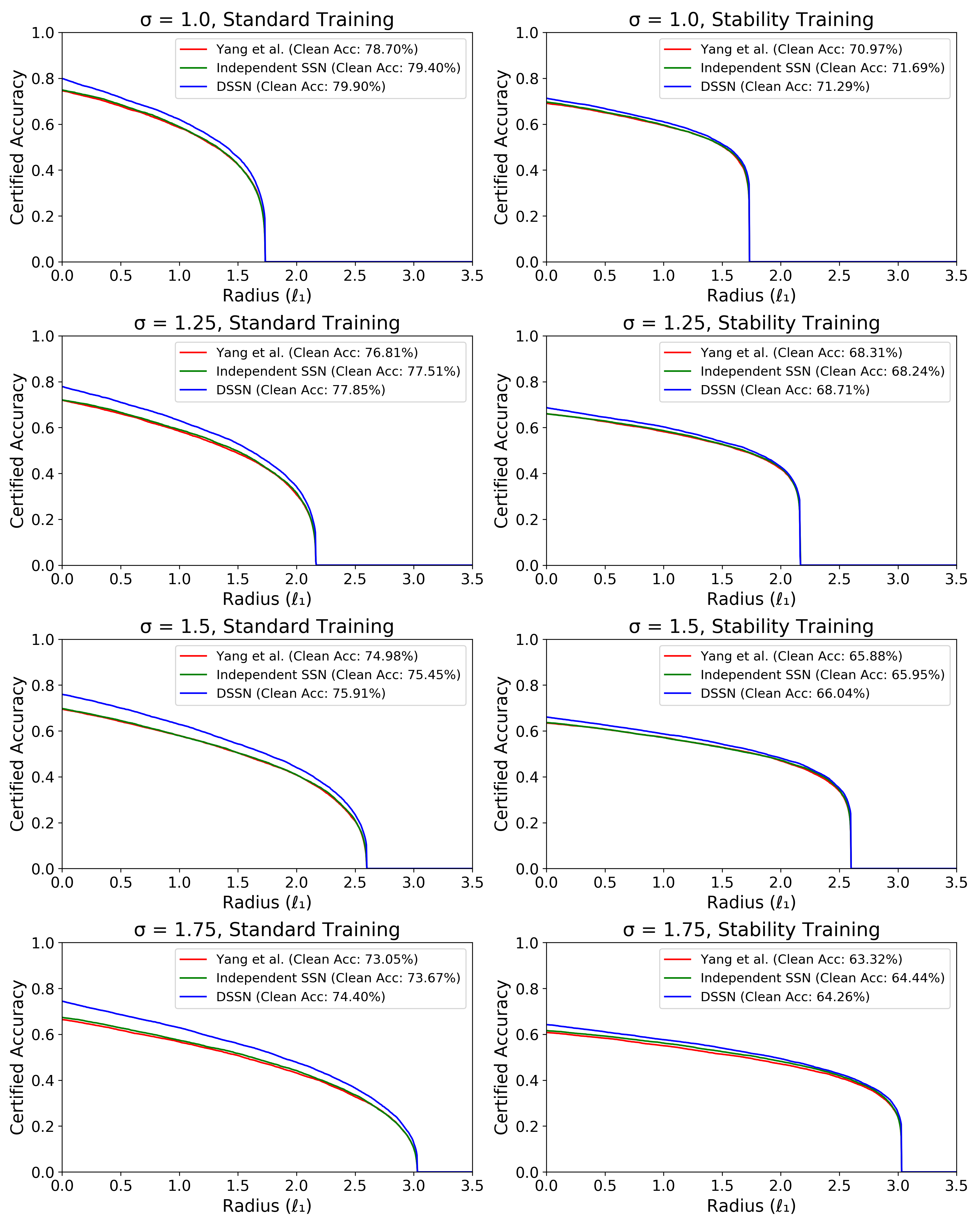}
    \caption{Certification results for CIFAR-10, comparing uniform additive noise, randomized SSN with independent noise, and DSSN, for $\sigma \in \{1.0,1.25,1.5,1.75\}$}
    \label{fig:appendix_cifar_1}.
\end{figure*}
\begin{figure*}
    \centering
    \includegraphics[width=\textwidth]{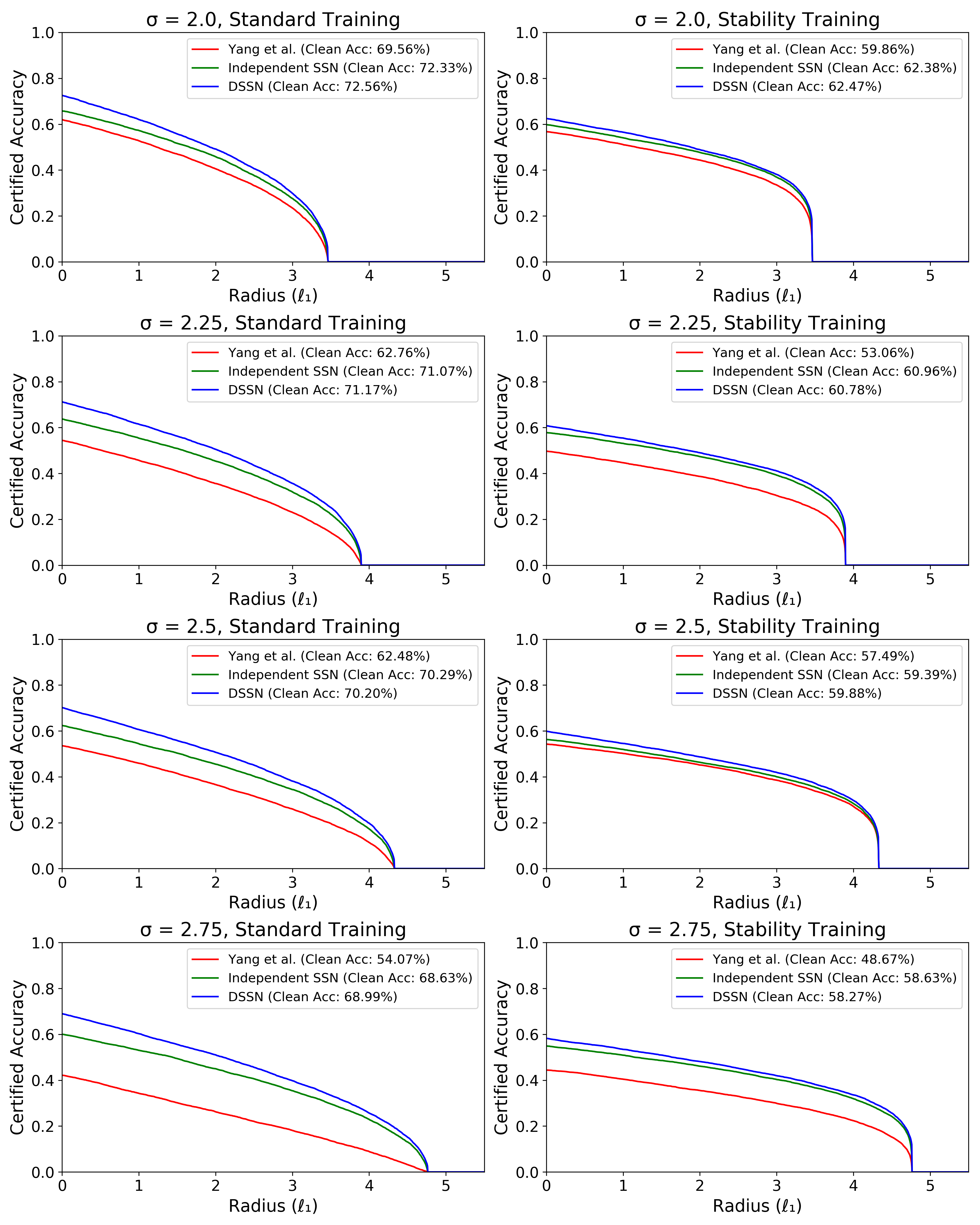}
    \caption{Certification results for CIFAR-10, comparing uniform additive noise, randomized SSN with independent noise, and DSSN, for $\sigma \in \{2.0,2.25,2.5,2.75\}$}
    \label{fig:appendix_cifar_2}.
\end{figure*}
\begin{figure*}
    \centering
    \includegraphics[width=\textwidth]{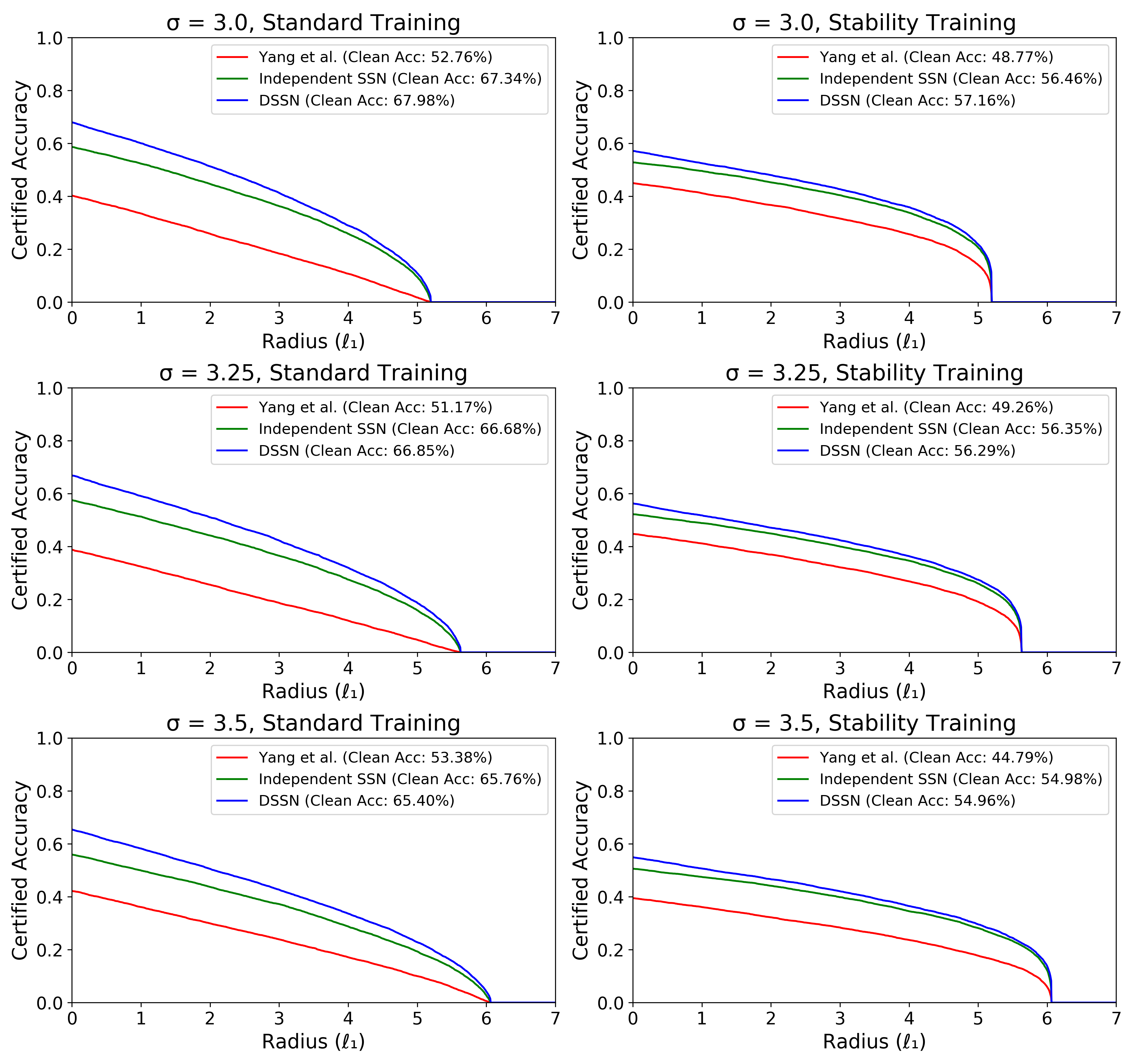}
    \caption{Certification results for CIFAR-10, comparing uniform additive noise, randomized SSN with independent noise, and DSSN, for $\sigma \in \{3.0,3.25,3.5\}$}
    \label{fig:appendix_cifar_3}.
\end{figure*}
\begin{figure*}
    \centering
    \includegraphics[width=\textwidth]{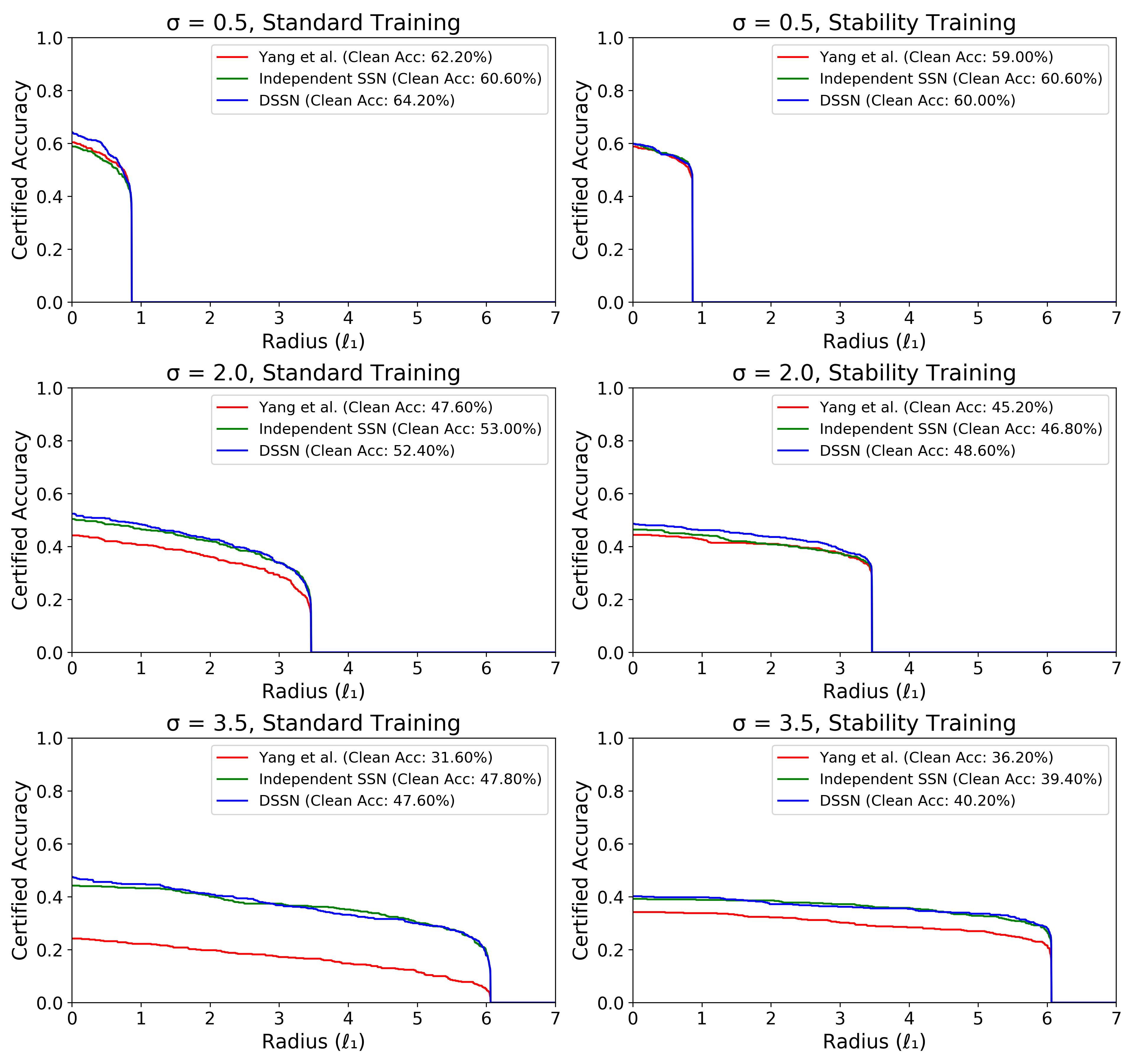}
    \caption{Certification results for ImageNet, comparing uniform additive noise, randomized SSN with independent noise, and DSSN, for $\sigma \in \{0.5,2.0,3.5\}$. Note that we see less improvement in reported certified accuracies due to derandomization (i.e., less difference between Independent SSN and DSSN) in ImageNet compared to in CIFAR-10, particularly at large noise levels.}
    \label{fig:appendix_imagenet_0}.
\end{figure*}

%% file: main.bbl
\begin{thebibliography}{36}
\providecommand{\natexlab}[1]{#1}
\providecommand{\url}[1]{\texttt{#1}}
\expandafter\ifx\csname urlstyle\endcsname\relax
  \providecommand{\doi}[1]{doi: #1}\else
  \providecommand{\doi}{doi: \begingroup \urlstyle{rm}\Url}\fi

\bibitem[Anil et~al.(2019)Anil, Lucas, and Grosse]{pmlr-v97-anil19a}
Anil, C., Lucas, J., and Grosse, R.
\newblock Sorting out {L}ipschitz function approximation.
\newblock In Chaudhuri, K. and Salakhutdinov, R. (eds.), \emph{Proceedings of
  the 36th International Conference on Machine Learning}, volume~97 of
  \emph{Proceedings of Machine Learning Research}, pp.\  291--301. PMLR, 09--15
  Jun 2019.
\newblock URL \url{http://proceedings.mlr.press/v97/anil19a.html}.

\bibitem[Carlini \& Wagner(2017)Carlini and Wagner]{carlini2017towards}
Carlini, N. and Wagner, D.
\newblock Towards evaluating the robustness of neural networks.
\newblock In \emph{2017 ieee symposium on security and privacy (sp)}, pp.\
  39--57. IEEE, 2017.

\bibitem[Chen et~al.(2018)Chen, Sharma, Zhang, Yi, and Hsieh]{chen2018ead}
Chen, P.-Y., Sharma, Y., Zhang, H., Yi, J., and Hsieh, C.-J.
\newblock Ead: elastic-net attacks to deep neural networks via adversarial
  examples.
\newblock In \emph{Proceedings of the AAAI Conference on Artificial
  Intelligence}, volume~32, 2018.

\bibitem[Chiang et~al.(2020)Chiang, Ni, Abdelkader, Zhu, Studor, and
  Goldstein]{Chiang2020Certified}
Chiang, P., Ni, R., Abdelkader, A., Zhu, C., Studor, C., and Goldstein, T.
\newblock Certified defenses for adversarial patches.
\newblock In \emph{International Conference on Learning Representations}, 2020.
\newblock URL \url{https://openreview.net/forum?id=HyeaSkrYPH}.

\bibitem[Cohen et~al.(2019)Cohen, Rosenfeld, and Kolter]{pmlr-v97-cohen19c}
Cohen, J., Rosenfeld, E., and Kolter, Z.
\newblock Certified adversarial robustness via randomized smoothing.
\newblock In Chaudhuri, K. and Salakhutdinov, R. (eds.), \emph{Proceedings of
  the 36th International Conference on Machine Learning}, volume~97 of
  \emph{Proceedings of Machine Learning Research}, pp.\  1310--1320, Long
  Beach, California, USA, 09--15 Jun 2019. PMLR.
\newblock URL \url{http://proceedings.mlr.press/v97/cohen19c.html}.

\bibitem[Fischer et~al.(2020)Fischer, Baader, and
  Vechev]{fischer2020randomized}
Fischer, M., Baader, M., and Vechev, M.
\newblock Certified defense to image transformations via randomized smoothing.
\newblock \emph{Advances in Neural Information Processing Systems Foundation
  (NeurIPS)}, 2020.

\bibitem[Goodfellow et~al.(2014)Goodfellow, Shlens, and
  Szegedy]{goodfellow2014explaining}
Goodfellow, I.~J., Shlens, J., and Szegedy, C.
\newblock Explaining and harnessing adversarial examples.
\newblock \emph{arXiv preprint arXiv:1412.6572}, 2014.

\bibitem[Gowal et~al.(2018)Gowal, Dvijotham, Stanforth, Bunel, Qin, Uesato,
  Mann, and Kohli]{gowal2018effectiveness}
Gowal, S., Dvijotham, K., Stanforth, R., Bunel, R., Qin, C., Uesato, J., Mann,
  T., and Kohli, P.
\newblock On the effectiveness of interval bound propagation for training
  verifiably robust models.
\newblock \emph{arXiv preprint arXiv:1810.12715}, 2018.

\bibitem[Jeong \& Shin(2020)Jeong and Shin]{NEURIPS2020_77330e13}
Jeong, J. and Shin, J.
\newblock Consistency regularization for certified robustness of smoothed
  classifiers.
\newblock In Larochelle, H., Ranzato, M., Hadsell, R., Balcan, M.~F., and Lin,
  H. (eds.), \emph{Advances in Neural Information Processing Systems},
  volume~33, pp.\  10558--10570. Curran Associates, Inc., 2020.
\newblock URL
  \url{https://proceedings.neurips.cc/paper/2020/file/77330e1330ae2b086e5bfcae50d9ffae-Paper.pdf}.

\bibitem[Kao et~al.(2020)Kao, Ko, and Lu]{kao2020deterministic}
Kao, C.-C., Ko, J.-B., and Lu, C.-S.
\newblock Deterministic certification to adversarial attacks via bernstein
  polynomial approximation, 2020.

\bibitem[Lecuyer et~al.(2019)Lecuyer, Atlidakis, Geambasu, Hsu, and
  Jana]{lecuyer2019certified}
Lecuyer, M., Atlidakis, V., Geambasu, R., Hsu, D., and Jana, S.
\newblock Certified robustness to adversarial examples with differential
  privacy.
\newblock In \emph{2019 IEEE Symposium on Security and Privacy (SP)}, pp.\
  656--672. IEEE, 2019.

\bibitem[Lee et~al.(2019)Lee, Yuan, Chang, and Jaakkola]{lee2019tight}
Lee, G.-H., Yuan, Y., Chang, S., and Jaakkola, T.
\newblock Tight certificates of adversarial robustness for randomly smoothed
  classifiers.
\newblock In \emph{Advances in Neural Information Processing Systems}, pp.\
  4910--4921, 2019.

\bibitem[Levine \& Feizi(2020{\natexlab{a}})Levine and
  Feizi]{DBLP:conf/nips/0001F20a}
Levine, A. and Feizi, S.
\newblock (de)randomized smoothing for certifiable defense against patch
  attacks.
\newblock In Larochelle, H., Ranzato, M., Hadsell, R., Balcan, M., and Lin, H.
  (eds.), \emph{Advances in Neural Information Processing Systems 33: Annual
  Conference on Neural Information Processing Systems 2020, NeurIPS 2020,
  December 6-12, 2020, virtual}, 2020{\natexlab{a}}.

\bibitem[Levine \& Feizi(2020{\natexlab{b}})Levine and
  Feizi]{levine2020robustness}
Levine, A. and Feizi, S.
\newblock Robustness certificates for sparse adversarial attacks by randomized
  ablation.
\newblock In \emph{Proceedings of the AAAI Conference on Artificial
  Intelligence}, volume~34, pp.\  4585--4593, 2020{\natexlab{b}}.

\bibitem[Levine \& Feizi(2020{\natexlab{c}})Levine and
  Feizi]{levine2020wasserstein}
Levine, A. and Feizi, S.
\newblock Wasserstein smoothing: Certified robustness against wasserstein
  adversarial attacks.
\newblock In \emph{International Conference on Artificial Intelligence and
  Statistics (AISTATS)}, 2020{\natexlab{c}}.

\bibitem[Levine \& Feizi(2021)Levine and Feizi]{levine2021deep}
Levine, A. and Feizi, S.
\newblock Deep partition aggregation: Provable defenses against general
  poisoning attacks.
\newblock In \emph{International Conference on Learning Representations}, 2021.
\newblock URL \url{https://openreview.net/forum?id=YUGG2tFuPM}.

\bibitem[Levine et~al.(2019)Levine, Singla, and Feizi]{levine2019certifiably}
Levine, A., Singla, S., and Feizi, S.
\newblock Certifiably robust interpretation in deep learning.
\newblock \emph{arXiv preprint arXiv:1905.12105}, 2019.

\bibitem[Li et~al.(2019{\natexlab{a}})Li, Chen, Wang, and
  Carin]{li2019certified}
Li, B., Chen, C., Wang, W., and Carin, L.
\newblock Certified adversarial robustness with additive noise.
\newblock In \emph{Advances in Neural Information Processing Systems}, pp.\
  9464--9474, 2019{\natexlab{a}}.

\bibitem[Li et~al.(2020)Li, Qi, Xie, and Li]{li2020sok}
Li, L., Qi, X., Xie, T., and Li, B.
\newblock Sok: Certified robustness for deep neural networks.
\newblock \emph{arXiv preprint arXiv:2009.04131}, 2020.

\bibitem[Li et~al.(2019{\natexlab{b}})Li, Haque, Anil, Lucas, Grosse, and
  Jacobsen]{Li2019PreventingGA}
Li, Q., Haque, S., Anil, C., Lucas, J., Grosse, R., and Jacobsen, J.
\newblock Preventing gradient attenuation in lipschitz constrained
  convolutional networks.
\newblock In \emph{NeurIPS}, 2019{\natexlab{b}}.

\bibitem[Mohapatra et~al.(2020)Mohapatra, Ko, Weng, Chen, Liu, and
  Daniel]{mohapatra2020higher}
Mohapatra, J., Ko, C.-Y., Weng, T.-W., Chen, P.-Y., Liu, S., and Daniel, L.
\newblock Higher-order certification for randomized smoothing.
\newblock \emph{Advances in Neural Information Processing Systems}, 33, 2020.

\bibitem[Raghunathan et~al.(2018)Raghunathan, Steinhardt, and
  Liang]{Raghunathan2018}
Raghunathan, A., Steinhardt, J., and Liang, P.
\newblock Semidefinite relaxations for certifying robustness to adversarial
  examples.
\newblock In \emph{Proceedings of the 32nd International Conference on Neural
  Information Processing Systems}, NIPS’18, pp.\  10900–10910, Red Hook,
  NY, USA, 2018. Curran Associates Inc.

\bibitem[Rosenfeld et~al.(2020)Rosenfeld, Winston, Ravikumar, and
  Kolter]{rosenfeld2020certified}
Rosenfeld, E., Winston, E., Ravikumar, P., and Kolter, Z.
\newblock Certified robustness to label-flipping attacks via randomized
  smoothing.
\newblock In \emph{International Conference on Machine Learning}, pp.\
  8230--8241. PMLR, 2020.

\bibitem[Salman et~al.(2019)Salman, Li, Razenshteyn, Zhang, Zhang, Bubeck, and
  Yang]{salman2019provably}
Salman, H., Li, J., Razenshteyn, I., Zhang, P., Zhang, H., Bubeck, S., and
  Yang, G.
\newblock Provably robust deep learning via adversarially trained smoothed
  classifiers.
\newblock In \emph{Advances in Neural Information Processing Systems}, pp.\
  11292--11303, 2019.

\bibitem[Salman et~al.(2020)Salman, Sun, Yang, Kapoor, and
  Kolter]{salman2020denoised}
Salman, H., Sun, M., Yang, G., Kapoor, A., and Kolter, J.~Z.
\newblock Denoised smoothing: A provable defense for pretrained classifiers.
\newblock \emph{Advances in Neural Information Processing Systems}, 33, 2020.

\bibitem[Singla \& Feizi(2020)Singla and Feizi]{pmlr-v119-singla20a}
Singla, S. and Feizi, S.
\newblock Second-order provable defenses against adversarial attacks.
\newblock In III, H.~D. and Singh, A. (eds.), \emph{Proceedings of the 37th
  International Conference on Machine Learning}, volume 119 of
  \emph{Proceedings of Machine Learning Research}, pp.\  8981--8991. PMLR,
  13--18 Jul 2020.
\newblock URL \url{http://proceedings.mlr.press/v119/singla20a.html}.

\bibitem[Szegedy et~al.(2013)Szegedy, Zaremba, Sutskever, Bruna, Erhan,
  Goodfellow, and Fergus]{szegedy2013intriguing}
Szegedy, C., Zaremba, W., Sutskever, I., Bruna, J., Erhan, D., Goodfellow, I.,
  and Fergus, R.
\newblock Intriguing properties of neural networks.
\newblock \emph{arXiv preprint arXiv:1312.6199}, 2013.

\bibitem[Teng et~al.(2020)Teng, Lee, and Yuan]{teng2020ell}
Teng, J., Lee, G.-H., and Yuan, Y.
\newblock {\$}{\textbackslash}ell{\_}1{\$} adversarial robustness certificates:
  a randomized smoothing approach, 2020.
\newblock URL \url{https://openreview.net/forum?id=H1lQIgrFDS}.

\bibitem[Tjeng et~al.(2019)Tjeng, Xiao, and Tedrake]{tjeng2018evaluating}
Tjeng, V., Xiao, K.~Y., and Tedrake, R.
\newblock Evaluating robustness of neural networks with mixed integer
  programming.
\newblock In \emph{International Conference on Learning Representations}, 2019.
\newblock URL \url{https://openreview.net/forum?id=HyGIdiRqtm}.

\bibitem[Weber et~al.(2020)Weber, Xu, Karlas, Zhang, and Li]{weber2020rab}
Weber, M., Xu, X., Karlas, B., Zhang, C., and Li, B.
\newblock Rab: Provable robustness against backdoor attacks.
\newblock \emph{arXiv preprint arXiv:2003.08904}, 2020.

\bibitem[Wong \& Kolter(2018)Wong and Kolter]{wong2018provable}
Wong, E. and Kolter, Z.
\newblock Provable defenses against adversarial examples via the convex outer
  adversarial polytope.
\newblock In \emph{International Conference on Machine Learning}, pp.\
  5283--5292, 2018.

\bibitem[Xiang et~al.(2020)Xiang, Bhagoji, Sehwag, and
  Mittal]{Xiang2020PatchGuardPD}
Xiang, C., Bhagoji, A., Sehwag, V., and Mittal, P.
\newblock Patchguard: Provable defense against adversarial patches using masks
  on small receptive fields.
\newblock \emph{ArXiv}, abs/2005.10884, 2020.

\bibitem[Yang et~al.(2020)Yang, Duan, Hu, Salman, Razenshteyn, and
  Li]{pmlr-v119-yang20c}
Yang, G., Duan, T., Hu, J.~E., Salman, H., Razenshteyn, I., and Li, J.
\newblock Randomized smoothing of all shapes and sizes.
\newblock In \emph{Proceedings of the 37th International Conference on Machine
  Learning}, pp.\  10693--10705, 2020.

\bibitem[Zhai et~al.(2020)Zhai, Dan, He, Zhang, Gong, Ravikumar, Hsieh, and
  Wang]{Zhai2020MACER}
Zhai, R., Dan, C., He, D., Zhang, H., Gong, B., Ravikumar, P., Hsieh, C.-J.,
  and Wang, L.
\newblock Macer: Attack-free and scalable robust training via maximizing
  certified radius.
\newblock In \emph{International Conference on Learning Representations}, 2020.
\newblock URL \url{https://openreview.net/forum?id=rJx1Na4Fwr}.

\bibitem[Zhang et~al.(2018)Zhang, Weng, Chen, Hsieh, and
  Daniel]{NEURIPS2018_d04863f1}
Zhang, H., Weng, T.-W., Chen, P.-Y., Hsieh, C.-J., and Daniel, L.
\newblock Efficient neural network robustness certification with general
  activation functions.
\newblock In Bengio, S., Wallach, H., Larochelle, H., Grauman, K.,
  Cesa-Bianchi, N., and Garnett, R. (eds.), \emph{Advances in Neural
  Information Processing Systems}, volume~31, pp.\  4939--4948. Curran
  Associates, Inc., 2018.
\newblock URL
  \url{https://proceedings.neurips.cc/paper/2018/file/d04863f100d59b3eb688a11f95b0ae60-Paper.pdf}.

\bibitem[Zhang et~al.(2020)Zhang, Yuan, McCoyd, and Wagner]{Zhang2020ClippedBD}
Zhang, Z., Yuan, B., McCoyd, M., and Wagner, D.
\newblock Clipped bagnet: Defending against sticker attacks with clipped
  bag-of-features.
\newblock \emph{2020 IEEE Security and Privacy Workshops (SPW)}, pp.\  55--61,
  2020.

\end{thebibliography}
